\documentclass[12pt]{alt2022}

\title[Decentralized Cooperative RL with Hierarchical Information Structure]{Decentralized Cooperative Reinforcement Learning with \\ Hierarchical Information Structure}
\usepackage{times}
\usepackage{hyperref}

\altauthor{%
 \Name{Hsu Kao} \Email{hsukao@umich.edu}\\
 \addr University of Michigan
 \AND
 \Name{Chen-Yu Wei} \Email{chenyu.wei@usc.edu}\\
 \addr University of Southern California
 \AND
 \Name{Vijay Subramanian} \Email{vgsubram@umich.edu}\\
 \addr University of Michigan
}

\usepackage{comment}
\usepackage{mathtools}
\usepackage{enumitem}
\usepackage{epstopdf}
\ifpdf
  \DeclareGraphicsExtensions{.eps,.pdf,.png,.jpg}
\else
  \DeclareGraphicsExtensions{.eps}
\fi
\usepackage{easybmat}
\usepackage{arydshln}
\usepackage{xcolor}

\makeatletter
\newcommand{\eqnum}{\leavevmode\hfill\refstepcounter{equation}\textup{\tagform@{\theequation}}}
\makeatother
\newcounter{protocol}


\newtheorem{thm}{Theorem}

\newtheorem{cor}[thm]{Corollary}
\newtheorem{lem}[thm]{Lemma}

\newtheorem{ass}[thm]{Assumption}

\usepackage{prettyref}
\newcommand{\pref}[1]{\prettyref{#1}}
\newcommand{\savehyperref}[2]{\texorpdfstring{\hyperref[#1]{#2}}{#2}}
\newrefformat{thm}{\savehyperref{#1}{Theorem~\ref*{#1}}}
\newrefformat{prop}{\savehyperref{#1}{Proposition~\ref*{#1}}}
\newrefformat{cor}{\savehyperref{#1}{Corollary~\ref*{#1}}}
\newrefformat{lem}{\savehyperref{#1}{Lemma~\ref*{#1}}}
\newrefformat{fact}{\savehyperref{#1}{Fact~\ref*{#1}}}
\newrefformat{conj}{\savehyperref{#1}{Conjecture~\ref*{#1}}}
\newrefformat{ass}{\savehyperref{#1}{Assumption~\ref*{#1}}}
\newrefformat{ppt}{\savehyperref{#1}{Property~\ref*{#1}}}
\newrefformat{dff}{\savehyperref{#1}{Definition~\ref*{#1}}}
\newrefformat{cond}{\savehyperref{#1}{Condition~\ref*{#1}}}
\newrefformat{ex}{\savehyperref{#1}{Example~\ref*{#1}}}
\newrefformat{rmk}{\savehyperref{#1}{Remark~\ref*{#1}}}
\newrefformat{alg}{\savehyperref{#1}{Algorithm~\ref*{#1}}}
\newrefformat{ptc}{\savehyperref{#1}{Protocol~\ref*{#1}}}
\newrefformat{sec}{\savehyperref{#1}{Section~\ref*{#1}}}
\newrefformat{app}{\savehyperref{#1}{Appendix~\ref*{#1}}}
\newrefformat{fig}{\savehyperref{#1}{Figure~\ref*{#1}}}
\newrefformat{tab}{\savehyperref{#1}{Table~\ref*{#1}}}
\newrefformat{line}{\savehyperref{#1}{Line~\ref*{#1}}}
\usepackage{chngcntr}
\usepackage{apptools}

   \def\dl{\delta}

\def\Ic{\mathcal{I}}

\def\E{\mathbb{E}}   
\def\I{\mathbb{I}}   
   \def\P{\mathbb{P}}
 
\def\aB{\mathbf{a}}

 \def\Rf{\mathfrak{R}}

\def\muh{\hat{\mu}}

\usepackage{stackengine}
\stackMath
\newcommand\tsup[2][2]{%
 \def\useanchorwidth{T}%
  \ifnum#1>1%
    \stackon[-.5pt]{\tsup[\numexpr#1-1\relax]{#2}}{\scriptscriptstyle\sim}%
  \else%
    \stackon[.5pt]{#2}{\scriptscriptstyle\sim}%
  \fi%
}

\newlength{\dhatheight}

\def\es{\enspace}

\def\1{\mathbf{1}}
\def\0{\mathbf{0}}
\def\eqsp{{\hspace{1.35em}}}

\newcommand{\rn}[1]{\textup{\lowercase\expandafter{\romannumeral#1}}}

\def\bms{\begin{bmatrix}}
\def\bme{\end{bmatrix}}
\def\beq{\begin{equation}}
\def\eeq{\end{equation}}
\def\bal{\begin{equation}\begin{aligned}}
\def\eal{\end{aligned}\end{equation}}
\def\bals{\begin{equation*}\begin{aligned}}
\def\eals{\end{aligned}\end{equation*}}

\newcommand\numberthis{\addtocounter{equation}{1}\tag{\theequation}}

\usepackage{xspace}
\usepackage{amsmath}
\usepackage{graphicx}
\usepackage{framed}
\usepackage{bbm}
\usepackage{algorithm} 
\usepackage{mathrsfs}
\usepackage{amssymb}
\usepackage{dashrule}
\usepackage{bm}
\usepackage{makecell}
\usepackage{tabulary}
\usepackage{xcolor}
\usepackage{prettyref}
\usepackage{mathtools}
\usepackage{amsmath}
\usepackage{multirow}
\usepackage{nicefrac}
\usepackage{amssymb}
\usepackage{pifont}
\definecolor{Green}{rgb}{0.13, 0.65, 0.3}
\usepackage{colortbl}
\usepackage{algpseudocode}
\newcommand{\otil}{\widetilde{\order}}

\DeclareMathOperator*{\argmax}{argmax} 

\newcommand{\calA}{\mathcal{A}}
\newcommand{\calB}{\mathcal{B}}

\newcommand{\calF}{\mathcal{F}}

\newcommand{\calS}{\mathcal{S}}

\newcommand{\calV}{\mathcal{V}}

\newcommand{\alg}{\scalebox{0.9}{\textsf{ALG}}\xspace}

\newcommand{\hatP}{\hat{P}}

\newcommand{\hatp}{\hat{p}}

\newcommand{\hatR}{\hat{R}}

\newcommand{\term}{\textbf{term}}

\newcommand{\order}{\mathcal{O}}

\newcommand{\one}{\mathbf{1}}

\newcommand{\bns}{\textup{\textsf{bns}}}
\newcommand{\ri}{r_{t_i,h}}
\newcommand{\bi}{b_{t_i,h}}
\newcommand{\si}{s_{t_i, h+1}}

\usepackage{tikz}

\newcommand{\nonl}{\renewcommand{\nl}{\let\nl}}

\newcommand{\Reg}{\textup{Reg}\xspace}

\usepackage{caption}

\newcommand{\leader}{\text{Leader}\xspace}
\newcommand{\follower}{\text{Follower}\xspace}

\begin{document}

\maketitle

\begin{abstract}
Multi-agent reinforcement learning (MARL) problems are challenging due to information asymmetry. To overcome this challenge, existing methods often require high level of coordination or communication between the agents. We consider two-agent multi-armed bandits (MABs) and Markov decision processes (MDPs) with a hierarchical information structure arising in applications, which we exploit to propose simpler and more efficient algorithms that require no coordination or communication. In the structure, in each step the ``leader" chooses her action first, and then the ``follower" decides his action after observing the leader's action. The two agents observe the same reward (and the same state transition in the MDP setting) that depends on their joint action. For the bandit setting, we propose a hierarchical bandit algorithm that achieves a near-optimal gap-independent regret of $\otil(\sqrt{ABT})$ and a near-optimal gap-dependent regret of $\order(\log(T))$, where $A$ and $B$ are the numbers of actions of the leader and the follower, respectively, and $T$ is the number of steps. We further extend to the case of multiple followers and the case with a deep hierarchy, where we both obtain near-optimal regret bounds. For the MDP setting, we obtain $\otil(\sqrt{H^7S^2ABT})$ regret, where $H$ is the number of steps per episode, $S$ is the number of states, $T$ is the number of episodes. This matches the existing lower bound in terms of $A, B$, and $T$. 
\end{abstract}


\section{Introduction}\label{sec:intro}

Multi-agent reinforcement learning (MARL) has received great attention due to its wide variety of applications and the tremendous advances in single-agent RL techniques \citep{Zhang_MARLSurvey_2019}. In a multi-agent environment, each agent has different observations and may have different sets of information. This is referred to as the \emph{information asymmetry} property \citep{Chang_MAMAB_2021}. One straightforward method used with information asymmetry is to let every agent concurrently learn based on its own information using single-agent algorithms. However, this creates the \emph{non-stationarity issue} since the effective environment observed by each agent is time-varying, which sometimes causes non-convergence of the algorithms. Another line of solutions is to enforce coordination among agents, essentially transforming a multi-agent system back to (or making it more similar to) a single-agent one. One way to achieve this is through communication \citep{Shahrampour_MAMAB_2017,Zhang_NetMARLSurvey_2019}, which introduces extra costs that may be intolerable in some cases. A more broadly applicable scheme is through the common information (CI) approach \citep{Demos_CI_2013, Chang_MAMAB_2021,Dibangoye_CentralLearn_2018}. The CI approach relies on a set of CI shared by all agents, and all agents need to agree on a protocol that specifies the joint policy updates of all agents upon receiving a certain piece of CI. With this protocol, an agent may be able to infer the actions taken by other agents without observing them. However, this approach has several shortcomings, making it hard to apply in practice: it has high computational complexity (since every agent has to perform policy updates for all the other agents), and it requires all agents to have very tight coordination (e.g., sharing randomization seeds in each round and knowing all details in the algorithms of other agents), which may be infeasible if synchronization among agents or agent privacy is an issue. 


In this paper, we address the aforementioned issues in a special but widely applicable MARL setting. We consider MARL team problems with a particular hierarchical information structure between agents under the settings of multi-agent multi-armed bandits (MAMABs) and multi-agent Markov decision processes (MDPs). In this structure, decisions are made sequentially, and a decision maker has all the decisions from decision makers that act before it in the sequence. In the two-agent case, one of the agents (the ``leader") chooses her action first, while the other agent (the ``follower") chooses his action after observing the leader's action. This setting is similar to the Stackelberg game but with the players cooperating to achieve the same objective. Such hierarchical information structure arises in many applications. For example, in a cognitive radio (CR) wireless network, the primary user (PU) first decides its resource allocation scheme; then based on this scheme the secondary user (SU) chooses its own resource allocation scheme that minimally interferes with the PU's transmission \citep{Ning_CRStackelberg_2020}. While this problem can also be solved using the CI approach or other MAMAB algorithms \citep{Kalathil_Matching_2014,Chang_MAMAB_2021}, as discussed earlier, they are intensive either in computation or in communication. In this work, we exploit the hierarchical information structure, and propose simpler and more efficient MARL algorithms that require neither communication nor explicit coordination, while achieving near-optimal regret bounds. Such algorithms could be much easier to deploy in practice.

In more detail, we first consider the two-agent bandit setting, where both agents observe the same reward determined by their \emph{joint action}, but only the follower observes the leader's action but not \emph{vice versa}. For this setting, we propose a decentralized algorithm that achieves a near-optimal gap-independent regret bound of $\otil(\sqrt{ABT})$\footnote{We use $\otil(\cdot)$ to hide poly-logarithmic factors.} and a gap-dependent bound of $\order(\log(T))$, where $A$ and $B$ are the numbers of actions of the leader and the follower, respectively, and $T$ is the number of time-steps. In our method, the leader performs an Upper Confidence Bound (UCB)-based algorithm \citep{auer2002finite} with a modified bonus term related to the follower's regret bound, while the follower can use any algorithm that achieves sub-linear regret. Without explicit coordination, the agents perform joint exploration over the action space and learn with low regret.
Interestingly, our hierarchical bandit setting mathematically coincides with the bandit-over-bandit framework considered in model selection \citep{agarwal2017corralling, arora2021corralling, cutkosky2021dynamic}, although the two problems are motivated by very different applications. Our algorithm can be readily used for model selection, and our gap-dependent bound (\pref{thm: gap dependent bound}) answers the open question in \citep{arora2021corralling} by improving their regret bound by a factor of $\log(T)$ (to our knowledge, \citep{arora2021corralling} is the only work on model selection that achieves a logarithmic gap-dependent bound). We further extend our idea to two more complicated settings. The first is the case of multiple followers, where each follower only observes its own reward while the leader only observes the sum of the rewards of all followers. The other is the case with a deep hierarchy, where more than two agents make decisions sequentially based on the decisions made by prior agents, and all agents observe the same reward.  In both extensions, our algorithms also achieve near-optimal regret bounds. 

Next, we generalize the above idea to the two-agent MDP setting. In this setting, the state evolution and reward observable by both agents are sampled from distributions depending on the current state and the joint action of the agents; as before, the follower observes the leader's action but not vice versa. Similar to the bandit case, we propose a decentralized learning method that enables the agents to perform joint exploration without communication or explicit coordination. Our algorithm is based on an intriguing combination of two exploration strategies developed for single-agent reinforcement learning: UCB-H \citep{UCB-H_2018} and UCBVI \citep{UCBVI_2017}. 
By letting the leader execute a UCB-H-styled algorithm and the follower use a UCBVI-styled one, the agents jointly achieve a regret upper bound of $\otil(\sqrt{H^7S^2ABT})$ ($H$ is the horizon length, and $T$ is the number of episodes), while the regret lower bound is $\Omega(\sqrt{H^2SABT})$, inherited from the single-agent MDP setting \citep{UCBVI_2017}. Tightening our bound without sacrificing the benefit of decentralized learning is left as an open question.

\paragraph{Related work.}
Algorithms for various MAMAB settings have gained increasing interest recently, but there is only a limited literature that investigates the effects and challenges caused by information asymmetry in the setting where agents jointly interact with the environment as is common in MARL applications, with the corresponding MDP setting receiving even less attention. \cite{Chang_MAMAB_2021} study the MAMAB setting where the reward is determined by the joint action with three types of information asymmetry: unobserved actions and common rewards, observed actions and independent rewards, and unobserved actions and independent rewards; the first two settings can be solved by the notion of the CI approach\footnote{Their mUCB algorithm for the first setting is equivalently the CI approach in combination with the UCB1 algorithm.}, while in the last setting they propose an ``explore then commit" type algorithm that achieves an $\order\left(\log(T)\right)$ regret. \cite{Jin_Stackelberg_2021} consider sample-efficient learning in bandit games and bandit-RL games. Their bandit game corresponds to our hierarchical bandits, but under a general reward setup (i.e., in their setting, the rewards of the two agents have different means, unlike the common-reward setting we consider here). 
They consider centralized and offline learning assuming access to a sample generator, while we consider decentralized and online learning through interactions with the environment. While their results imply a worst case information-theoretic gap to the Stackelberg game value that cannot be closed, it is not the case in our team problem. On the other hand, our hierarchical MDP has a more general transition structure than their bandit-RL game (in their setting, only the follower is involved in an MDP). 
\cite{arora2021corralling} study meta-learning over bandit algorithms, which exhibits a similar mathematical structure to our hierarchical bandit problem, though from a very different perspective. Our hierarchical bandit algorithm can be readily used as an algorithm for their setting,  
and our gap-dependent bound improves their $\order(\log^2 (T))$ bound by a factor of $\log(T)$, resolving their question on the tightness of their result.

The readers may refer to \pref{app: more-relate} for more literature survey on the topic.

\section{Preliminaries}\label{sec:prelim}

We first define some notation. For a positive integer $n$, we denote $[n]=\{1, 2, \ldots, n\}$. For an integer $n$, we define $n^+=\max\{n, 1\}$. 

\paragraph{Two-agent hierarchical bandits.}
Consider a two-agent MAB where the rewards are decided by the joint action of the two agents U1 (leader) and U2 (follower). Let $A$ and $B$ be the numbers of actions (which are arms in the context) of U1 and U2, respectively, and $T$ be the number of time steps. Without loss of generality, we assume that $A, B\leq T$. Under the hierarchical information structure, in round $t\in[T]$, U1 first chooses an action $a_t\in [A]$; after observing U1's action $a_t$, U2 then chooses another action $b_t\in[B]$. However, U1 cannot observe U2's action $b_t$. These two actions jointly generate a noisy reward $r_t\in[0,1]$ with expectation $\mu_{a_t, b_t}$, and both agents observe\footnote{Our analysis can straightforwardly handle a more general case where U1 and U2 receive different (independent) noisy copies of the reward with the same mean. For simplicity, we assume that they receive the same copy.} $r_t$. 
 For ease of presentation, we assume without loss of generality that the best action of U2 given any choice of U1 is indexed by $1$, i.e., $\mu_{a,1}\geq \mu_{a,b}$ for all $a,b$; similarly, the best action of U1 is indexed by $1$, i.e., $\mu_{1,1}\geq \mu_{a,1}$ for all $a$. Then the (common) goal of the agents is to minimize the pseudo-regret defined as follows: 
\begin{align*}
    \Reg(T) = \sum_{t=1}^T (\mu_{1,1} - \mu_{a_t,b_t}).  
\end{align*}

\paragraph{Two-agent hierarchical MDPs.}
Consider a two-agent $H$-step finite-horizon MDP where the rewards and state transitions depend on the joint action of the two agents U1 and U2, with the process run over $T$ episodes. This generalizes the previous two-agent bandit setting. The state space is $\calS$, with a number of $S=|\calS|$ states. In each state, U1 and U2 choose actions from $[A]$ and $[B]$, respectively. We assume that $S, A, B, H$ are all upper bounded by $T$. Every episode $t$ starts with an initial state $s_{t,1}\in\calS$. In the $h$-th step of the $t$-th episode, the agents first observe $s_{t,h}\in \calS$. Under the hierarchical information structure, U1 chooses an action $a_{t,h}\in[A]$, followed by U2 choosing another action $b_{t,h}\in[B]$ upon seeing $a_{t,h}$. After the actions are chosen, both agents receive a reward $r_{t,h}\in[0,1]$ with $\E[r_{t,h}]=R(s_{t,h}, a_{t,h}, b_{t,h})$, and then the state transitions to the next state $s_{t,h+1}\sim P(\cdot|s_{t,h},a_{t,h},b_{t,h})$. The episode ends right after the state transitions to $s_{t,H+1}$. In the RL setting we consider, rewards are commonly observed by both agents, but they do not know the reward function $R$ or the transition probability $P$.

An $H$-step policy for U1 can be represented as $\pi^1=\{\pi^1_1, \ldots, \pi^1_H\}$, where $\pi_h^1: \calS\rightarrow [A]$ specifies the choice of her action on each state when she is at step $h$; a policy for U2 can be represented as $\pi^2=\{\pi^2_1, \ldots, \pi^2_H\}$, where $\pi_h^2: \calS\times [A]\rightarrow [B]$ specifies the choice of his action on each state and under each possible choice of U1, when he is at step $h$. We define the state value function at step $h$ under a policy pair $(\pi^1, \pi^2)$ as 
\begin{align*}
     V_h^{\pi^1, \pi^2}(s) = \scalebox{0.9}{$\displaystyle\E\left[ \sum_{k=h}^H R(s_{k}, a_k, b_k) ~\bigg|~ s_h = s, a_k=\pi_k^1(s_k), b_k = \pi_k^2(s_k, a_k), s_{k+1}\sim P(\cdot|s_k,a_k,b_k), \forall k\geq h \right]. $}
\end{align*}
with $V_{H+1}^{\pi^1, \pi^2}(\cdot)\triangleq 0$. Also, we define the state-action value function as 
\begin{align*}
	Q_h^{\pi^1, \pi^2}(s,a,b)= R(s, a, b) + \E \Bigg[ V_{h+1}^{\pi^1, \pi^2}(s_{h+1}) \bigg| s_{h+1} \sim P(\cdot | s, a, b) \Bigg].
\end{align*}
The optimal value functions are then given by $V_{*,h}(s)=\max_{\pi^1, \pi^2} V_h^{\pi^1, \pi^2}(s)$ and $Q_{*,h}(s, a, b)=\max_{\pi^1, \pi^2} Q_h^{\pi^1, \pi^2}(s, a, b)$. By dynamic programming, we have the following for all $h, s, a, b$: 
\begin{align*}
    V_{*, h}(s) = \max_{a,b} Q_{*,h}(s,a,b) \quad \text{and}\quad
    Q_{*,h}(s,a,b) = R(s,a,b) + \E_{s'\sim P(\cdot|s,a,b)}\left[V_{*,h+1}(s')\right], 
\end{align*}
with $V_{*,H+1}(\cdot)\triangleq 0$. 
We further define $Q_{*,h}(s,a) = \max_b Q_{*,h}(s,a,b)$.  With this notation, we can write the regret of the agents as
\begin{align*}
    \Reg(T) = \sum_{t=1}^T \left(V_{*,1}(s_{t,1}) - V^{\pi^1_t, \pi^2_t}_1(s_{t,1})\right). 
\end{align*}

\paragraph{Benchmark.}
In two-agent cases, there are three obvious types of information structure in terms of \emph{action information asymmetry}: the complete information setting, the no information setting, and the hierarchical setting considered in this paper. Using the CI approach, one may achieve the lower regret bounds of $\otil(\sqrt{ABT})$ for the bandit setting and $\otil(\sqrt{H^2SABT})$ for the MDP setting, with higher complexity and stronger assumptions. For details see \pref{app: CI-approach}.
\section{Learning Hierarchical Bandits}\label{sec: bandit}

Since U1 does not observe U2's actions, it is unclear how U1 can utilize or interpret the samples she receives. For example, if U1 receives a low reward, one possibility is that U1 has chosen a bad action, so whatever action U2 chooses, the reward is going to be low; but it is also possible that the action chosen by U1 is actually good (i.e., the reward would be high if U2 chose a good subsequent action), but U2 has chosen a bad subsequent action. If the identity of U2's action is not revealed, in general, U1 cannot distinguish these two cases. As mentioned in \pref{sec:intro}, this issue can be resolved by the CI approach, which enables U1 to infer the actions taken by U2, even if they are not directly observable.

We show that providing U2 executes a \emph{no-regret} algorithm (i.e., an algorithm that always guarantees a sub-linear-in-$T$ regret), we can actually make the agents converge to playing optimal actions while keeping U1 agnostic to U2's actions during the whole process. The key observation is that when U2 is a no-regret learning agent, his choices of action under a given action of U1 will converge to the best one, hence avoiding the second possibility mentioned above in the long run. 

The following \pref{ass: U2's algorithm} specifies the condition that U2's algorithm should satisfy, for which using an MAB algorithm with near-optimal, high-probability regret guarantee for each action $a\in[A]$ is enough. Some existing algorithms satisfying this assumption are: UCB1 \citep{auer2002finite}, KL-UCB \citep{garivier2011kl}, EXP3.P \citep{auer2002nonstochastic}, EXP3-IX \citep{neu2015explore}. \footnote{While in \pref{ass: U2's algorithm} we assume that U2 uses a near-optimal algorithm for MAB for ease of presentation, our framework can also handle the case where he uses some sub-optimal algorithm (the final regret bound will also be sub-optimal though).  }

\begin{ass}\label{ass: U2's algorithm}
    U2 guarantees the following for some universal constant $\kappa\geq 1$ with probability at least $1-\delta$: 
    \begin{align*}
        \forall t, a, \qquad \sum_{\tau=1}^{t} \I[a_\tau=a](\mu_{a,1} - \mu_{a, b_\tau}) \leq \sqrt{\kappa B\sum_{\tau=1}^t \I[a_\tau=a] \log(T/\delta)}.  
    \end{align*}
\end{ass}

With \pref{ass: U2's algorithm}, our algorithm is presented in \pref{alg: two-stage alg}. In \pref{alg: two-stage alg}, U2 simply executes the specified algorithm. On the other hand, U1 follows a selection rule that is very similar to the UCB1 algorithm of \cite{auer2002finite}. Specifically, in the beginning of each round $t$, U1 calculates the \emph{empirical mean} of each of her actions: 
\begin{align}
    \muh_t(a) = \frac{1}{n_t(a)^+}\sum_{\tau=1}^{t-1} \I[a_\tau=a] r_\tau, \qquad \text{where\ } n_t(a)=\sum_{\tau=1}^{t-1}\I[a_\tau=a]
    \label{eq: algorithm 1 specification}   
\end{align}
Also, U1 calculates the \emph{bonus} for action $a$, which is simply the \emph{average regret upper bound} of U2 on action $a$ up to time $t-1$ (by \pref{ass: U2's algorithm}, the regret on action $a$ up to time $t-1$ is upper bounded by $\sqrt{\kappa Bn_t(a)\log(T/\delta)}$), plus a term given by the Hoeffding bound (see \pref{line: U1 choice} of \pref{alg: two-stage alg}). Then U1 simply chooses the arm with the largest empirical mean plus bonus. While appearing similar, it is not the same as the standard UCB1 algorithm. First, the empirical mean $\muh_t(a)$ is an average over samples that are not identically distributed since each $r_t$ also depends on the action chosen by U2. Second, the construction of the bonus term involves the \emph{regret upper bound} of U2, in addition to another term implied by concentration inequalities. However, our algorithm is indeed inspired by the UCB1 algorithm --- in UCB1, the bonus term is designed to fulfill the following two properties: 1) for every arm $a$, the empirical mean of reward plus bonus should upper bound the true mean of reward with high probability; 2) the sum of bonus of the chosen arms in rounds $t=1, 2, \ldots, T$ should be sub-linear in $T$. With both properties, the UCB1 algorithm is guaranteed to have sub-linear regret. 

In our case, we shall identify $\mu_{a,1}$ as the ``true mean of reward'' of arm $a$. To satisfy the first property above, we would like to add a bonus term that upper bounds $\mu_{a,1} - \muh_t(a)$. Note that $\muh_t(a)$ is the mean of reward of U2 in the rounds when U1 chooses $a$, so $\mu_{a,1} - \muh_t(a)$ is simply the \emph{average regret} of U2 under U1's action $a$. Therefore, adding the regret bound of U2 as the bonus of U1 gives the first property above. The second property is also satisfied as long as U2 uses algorithms with sub-linear regret guarantees. 

We remark that it is important for U1 to use a slightly larger bonus (i.e., $\otil\big(\sqrt{B/n_t(a)}\big)$) rather than the standard one she would use when she plays alone (i.e., $\otil\big(\sqrt{1/n_t(a)}\big)$). This extra $\sqrt{B}$ factor makes the bonus of U1 decrease at a slower rate, leading to more exploration on each of U1's actions, and allowing U2 to have enough time to find the best action under each of U1's actions.  



\begin{algorithm}[t]
\caption{UCB for Hierarchical Bandits}\label{alg: two-stage alg}
\nl \textbf{define}: $c>0$ is a universal constant.   \\
\nl  U2 starts running algorithms satisfying \pref{ass: U2's algorithm} (or \pref{ass: U2's algorithm gap}) with some $\kappa\geq 1$ 
(we denote the instance of algorithm under action $a\in[A]$ as $\alg(a)$). \\
\nl  \For{$t=1, 2, \dots,T$}{
\nl \label{line: U1 choice} U1 chooses $ a_t\in\underset{a\in[A]}{\argmax}\ \muh_t(a)+\sqrt{\frac{\kappa B\log(T/\dl)}{n_t(a)^+}} + c\sqrt{\frac{\log(T/\delta)}{n_t(a)^+}}$.\  \  ($\muh_t(a)$, $n_t(a)$ are defined in \pref{eq: algorithm 1 specification})\\
\nl  After observing $a_t$, U2 calls $\alg(a_t)$, which outputs an action $b_t$. \\
\nl  U2 chooses $b_t$. \\
\nl  U1 and U2 observe $r_t$, and U2 updates $\alg(a_t)$ using $r_t$. 
}
\end{algorithm}



The analysis of this algorithm is straightforward given the above intuition, which we show in the following theorem.
\begin{theorem}\label{thm: main theorem for bandit}
    Suppose that U2 uses algorithms that satisfy \pref{ass: U2's algorithm}. Then \pref{alg: two-stage alg} guarantees that with probability at least $1-\order(\delta)$,  $\Reg(T) = \order\big(\sqrt{ABT\log(T/\delta)}\big)$. 
\end{theorem}
\begin{proof}
\allowdisplaybreaks
       \begin{align*}
         &\sum_{t=1}^T \left(\mu_{1,1} - \mu_{a_t,b_t}\right) \\
         &\leq \sum_{t=1}^T \left( \frac{1}{n_t(1)^+}\sum_{\tau=1}^{t-1}\I[a_\tau=1]\mu_{1,b_\tau} + \sqrt{\frac{\kappa B\log(T/\delta)}{n_t(1)^+}} - \mu_{a_t,b_t}\right) \tag{by \pref{ass: U2's algorithm} with $a=1$} \\
         &\leq \sum_{t=1}^T \left( \muh_t(1) + c\sqrt{\frac{\log(T/\delta)}{n_t(1)^+}} + \sqrt{\frac{\kappa B\log(T/\delta)}{n_t(1)^+}} - \mu_{a_t,b_t}\right) \tag{by \pref{lem: UCB concentration}} \\
         &\leq \sum_{t=1}^T \left( \muh_t(a_t) + c\sqrt{\frac{\log(T/\delta)}{n_t(a_t)^+}} +  \sqrt{\frac{\kappa B\log(T/\delta)}{n_t(a_t)^+}} - \mu_{a_t,b_t}\right)   \tag{by the selection rule of $a_t$} \\
         &= \sum_{t=1}^T \left( \muh_t(a_t) - \mu_{a_t,b_t}\right) + \order\left(\sqrt{ABT\log(T/\delta)}\right)  \\
         &= \underbrace{\sum_{a\in[A]}\sum_{t=1}^T \I[a_t=a]\left(\muh_t(a) -  \mu_{a,1} \right)}_{\term_1} + \underbrace{\sum_{a\in[A]}\sum_{t=1}^T \I[a_t=a]\left( \mu_{a,1} - \mu_{a_t,b_t}\right) }_{\term_2} + \order\left(\sqrt{ABT\log(T/\delta)}\right) 
    \end{align*}
    Notice that $\term_2=\order\left(\sum_{a\in[A]}\sqrt{n_{T+1}(a)\log(T/\delta)}\right) 
    = \order\left(\sqrt{ABT\log(T/\delta)}\right)$ due to \pref{ass: U2's algorithm}.  
    Besides, by Azuma's inequality, with probability at least $1-\order(\delta)$, for all $t$ and $a$,  
    \begin{align*}
        \muh_t(a) &= \frac{\sum_{\tau=1}^{t-1} \I[a_\tau=a]r_\tau }{n_t(a)^+}
        \leq \frac{\sum_{\tau=1}^{t-1} \I[a_\tau=a]\mu_{a,b_\tau}}{n_t(a)^+} + \order\left(\sqrt{\frac{\log(T/\delta)}{n_t(a)^+}}\right) \leq  \mu_{a,1} +  \order\left(\sqrt{\frac{\log(T/\delta)}{n_t(a)^+}}\right).  
    \end{align*}
    Therefore, 
    $
        \term_1 \leq \order\left(\sum_{a\in[A]}\sum_{t=1}^T \I[a_t=a] \sqrt{\frac{\log(T/\delta)}{n_t(a)^+}}\right) =\order\left(\sqrt{AT\log(T/\delta)}\right). 
    $
    Combining everything finishes the proof. 
\end{proof}

For MAB, there are also algorithms with refined gap-dependent regret bounds with only $\order(\log (T))$ dependence on $T$ (e.g., the UCB1 algorithm of \cite{auer2002finite}). Below we show that if U2 executes such algorithms, the overall regret can also be of order $\order(\log (T))$. Such algorithms satisfy the following assumption: 
\begin{ass}\label{ass: U2's algorithm gap}
    U2 guarantees the following for some universal constant $\kappa\geq 1$ with probability at least $1-\delta$:   
    \begin{align*}
        \forall t, a, \qquad \sum_{\tau=1}^{t} \I[a_\tau=a](\mu_{a,1} - \mu_{a,b_\tau}) \leq \min\left\{\kappa  \sum_{b\in\calB_a^{\times}}  \frac{ \log(T/\delta)}{\mu_{a,1}-\mu_{a,b}},\ \ \sqrt{\kappa B\sum_{\tau=1}^t \I[a_\tau=a] \log(T/\delta)} \right\}, 
    \end{align*}
    where $\calB_a^{\times}\triangleq \{b\in[B]: \mu_{a,1}-\mu_{a,b} > 0\}$ is the set of sub-optimal arms of U2 under U1's action $a$. 
\end{ass}
With \pref{ass: U2's algorithm gap}, \pref{alg: two-stage alg} has the following guarantee: 
\begin{theorem}\label{thm: gap dependent bound}
    Suppose that U2 uses algorithms that satisfy \pref{ass: U2's algorithm gap}. Then \pref{alg: two-stage alg} guarantees that with probability at least $1-\order(\delta)$,  
    \begin{align*}
        \Reg(T) = \order\left(\sum_{a\in\calA^{\times}} \frac{B\log (T/\delta)}{\mu_{1,1} - \mu_{a,1}} + \sum_{a\in \calA^{\circ}} \sum_{b\in\calB_a^{\times}} \frac{\log (T/\delta)}{\mu_{a,1}-\mu_{a,b}} \right),
    \end{align*}
    where $\calA^{\circ}\triangleq \{a\in[A]: \mu_{1,1}=\mu_{a,1}\}$,  $\calA^{\times}\triangleq \{a\in[A]: \mu_{1,1}-\mu_{a,1}>0\}=[A]\backslash \calA^{\circ}$, and $\calB_a^{\times}$ is as defined in \pref{ass: U2's algorithm gap}. 
\end{theorem}
The proof is deferred to \pref{app: mdp proofs}. In \pref{thm: gap dependent bound}, the regret consists of two parts. For an action $a\in[A]$ that is sub-optimal (i.e., $\mu_{1,1}>\mu_{a,1}$), the regret scales with $\frac{B}{\mu_{1,1}-\mu_{a,1}}$; for optimal ones (i.e., $\mu_{1,1}=\mu_{a,1}$), the regret scales with $\sum_{b\in\calB_{a}^{\times}}\frac{1}{\mu_{a,1}-\mu_{a,b}}$, i.e., the sum of inverse gaps of all U2's sub-optimal actions under $a$.  

Our hierarchical bandit setting coincides with the model selection problem studied in \cite{arora2021corralling}.  Our result in \pref{thm: gap dependent bound} improves their results in two ways. First, when using UCB-based algorithms as the base algorithm, our regret bound scales with $\log(T)$, while theirs scales with $\log^2(T)$. This answers their open question regarding whether $\log^2(T)$ is tight (see the discussion in their Section 4.1). Second, while they assume that U1's optimal action is unique (i.e., for all $a\neq 1$, $\mu_{1,1}>\mu_{a,1}$), we do not make such an assumption. 


\subsection{Extensions}\label{sec: extension}

\paragraph{Multiple Followers Case} 
Our framework can be easily extended to the case when there is a leader and multiple followers where each of the followers' rewards is only a function of the choice of the leader and the individual follower as well as independent across others, and the reward of the leader is an average/sum of that of all the followers. This is particularly useful in modeling networks with a star topology, e.g., federated learning systems (with leader being the server and followers being the clients), mobile networks (with leader being the base station/access point and followers being the mobile users), etc. In these networks, it is usually the case that the followers are heterogeneous and move into or out of the networks dynamically. Since our proposed method does not require per-round coordination and treats the algorithms of the followers as black boxes, the coordination overhead for the leader can be relatively low. Besides, our algorithm largely preserves privacy for the followers, which is also much more preferable than other schemes where the leader is required to know the algorithm of the followers. In \pref{app: multiple follower}, we show that near-optimal regret bounds can also be obtained in this case with an idea similar to that presented in \pref{sec: bandit}.

\paragraph{Deep Hierarchy Case} 
While in \pref{sec: bandit} we consider a two-layer model which only involves U1 and U2, the idea can be generalized to the case where there are $D>2$ layers. In other words, in each round, the decision is made jointly by $D$ agents with a fixed ordering, where agents making decision earlier cannot observe the actions taken by agents making decisions later. Such a protocol may be useful in modeling networks with deeper hierarchy, e.g., mobile networks where macro-, micro-, pico-, and femto- base stations are overlaid to serve user equipments \citep{jain2011, sigwele2020energy}.  In \pref{app:deep hie}, we design a multi-layer UCB algorithm for this setting, and show near-optimal regret bounds for it.

\section{Learning Hierarchical MDPs}


This setting is much more challenging than the hierarchical bandit setting. First, notice that in this setting, both agents are facing non-stationary transition and reward because of the dependence of these quantities on the policy of the other agent, which varies with time. Obtaining regret bounds in such time-varying MDPs is in general hard \citep{abbasi2013online, radanovic2019learning, tian2021online}, except for problems with special structures or extra assumptions \citep{radanovic2019learning, tian2021online, leonardos2021global} like our case here. Second, notice that in the bandit case, given any choice of U1, U2 is essentially facing a stationary MAB problem, and thus we can directly apply existing theorems for standard MAB; however, in the MDP case, the world that U2 sees on a certain step is still affected by the non-stationarity
of U1's policies in future steps. In this case, standard analysis for stationary MDPs cannot be directly applied. 

An initial idea to deal with this setting is to let both agents run existing UCB-based algorithms (e.g., UCBVI \citep{UCBVI_2017}, UCB-H \citep{UCB-H_2018}) with an increased bonus term for U1 to compensate for the regret of U2, imitating our hierarchical bandit solution. However, as we point out above, U2's world is also affected by the policies of U1 in future steps. Therefore, a natural solution is to do the opposite: let U2 run a standard algorithm over the non-stationary world he sees, and add extra bonus term to U2 in order to compensate for the regret of U1 in future steps. 

Unfortunately, for this hypothetical algorithm, it is unclear to us how to obtain a regret bound that is polynomial in the number of steps $H$. This is because by recursively adding extra bonus in each layer, 
we end up with a factor of $(AB)^{\nicefrac{H}{2}}$ in the regret bound, similar to the ``Deep Hierarchy Case'' discussed in \pref{sec: extension} and \pref{app:deep hie}.  
To address this issue, instead of trying to let U2 best respond to the non-stationary world created by U1, we exploit the fact that U2 has full knowledge about the joint action space, and let U2 find the best \emph{joint policy} of U1 and U2. Then U2 will execute his part of this joint policy even though U1 may not follow it. Although this brings other issues (discussed later), it avoids the need of U2 to compensate for the regret of U1 in later steps, and prevents the exponential blowup in the regret bound. 

Our algorithm for hierarchical MDPs is presented in \pref{alg: MDP alg}. To avoid cluttered notation, we drop the episode index $t$ when presenting the algorithm. Unlike \pref{alg: two-stage alg} where we can plug in any algorithm for U2 with the desired regret bound, in \pref{alg: MDP alg} we specify both agents' algorithms. We leave as an open problem how to design a black-box-reduction-styled algorithm for the MDP setting similar to \pref{alg: two-stage alg}. 

\setcounter{AlgoLine}{0}
\begin{algorithm}[t]
    \caption{UCB-H/UCBVI for Hierarchical MDP}\label{alg: MDP alg}
\nl  \textbf{define}: $\alpha_\tau=\frac{H+1}{H+\tau}$, $\bns^1_\tau = c'\sqrt{\frac{H^3SB\log(T/\delta)}{\tau^+}}$, $\bns^2_\tau= c\sqrt{\frac{H^2S\log(T/\delta)}{\tau^+}}$ where $c, c'\geq 1$ are universal constants.  \\
\nl  \textbf{initialize}:
     $Q_{h}^1(s,a), Q_{h}^2(s,a,b) \leftarrow H \ \ \forall h, s,a,b$,  \\
\nl   $n_h(s,a), n_h(s,a,b), n_h(s,a,b,s'), \theta_h(s,a,b)\leftarrow 0$\ \  $\forall h, s, a, b, s'$.  \\
\nl  \For{$t=1, \ldots, T$}{
\nl      U1 and U2 observes $s_1$. \\
\nl     \For{$h=1, \ldots, H$}{
            %
\nl U1 chooses $a_h\in\argmax_a Q_h^1(s_h, a)$. \\
\nl U2 observes $a_h$. \\
\nl \label{line: U2 selection rule}U2 chooses $b_{h}\in\argmax_{b} Q^2_{h}(s_{h}, a_{h}, b)$. \\
\nl U1 and U2 observe $r_{h}$ and $s_{h+1}$. \\
\nl U1 updates counts of visits: $n_h(s_h,a_h)\stackrel{+}{\leftarrow} 1$. \es (``$\stackrel{+}{\leftarrow}1$'' means to increase the number by $1$.) \\
\nl U2 updates counts of visits: $n_h(s_h,a_h,b_h) \stackrel{+}{\leftarrow} 1$, $n_h(s_h,a_h,b_h,s_{h+1})\stackrel{+}{\leftarrow} 1$. \\
\nl U2 updates cumulative reward: $\theta_h(s_h,a_h,b_h)\stackrel{+}{\leftarrow} r_h$.
        }
         \vspace*{2pt}\ \\
\nl        \textbf{U1 updates $Q/V$ functions} ($\approx$ UCB-H update rule): \\ 
\nl            \scalebox{1}{$V_{H+1}^1(\cdot)\leftarrow 0$}.\\
\nl            \For{$h=1, \ldots, H$}{
\hspace*{-0.5pt}\nl\hspace*{0.5pt}   \label{line: UCBQ Q update}             
                       $Q^1_{h}(s_h, a_h)\leftarrow (1-\alpha_\tau) Q^1_{h}(s_h, a_h) + \alpha_\tau \left(r_{h} + V^1_{h+1}(s_{h+1}) + \bns^1_\tau \right)$ \\
\hspace*{-0.5pt}\nl\hspace*{0.5pt}     \label{line: UCBQ V update}           
                       $V^1_h(s_h)\leftarrow \min\{\max_a Q^1_h(s_h, a), \ H\}$ \\
\hspace*{-0.5pt}\nl\hspace*{0.5pt}                
                       $\text{where\ } \tau=n_h(s_h,a_h)$. 
            } 
            \vspace*{2pt}\ \\
\nl           \textbf{U2 updates $Q/V$ functions} ($\approx$ UCBVI update rule): \\ 
\nl          Let $\hatP_h(s'|s,a,b) = \frac{n_h(s,a,b,s')}{n_h(s,a,b)}$ and $\hatR_h(s,a,b)=\frac{\theta_h(s,a,b)}{n_h(s,a,b)}\ \ \ \forall h, s,a,b,s'$. \\
\nl           (if $n_h(s,a,b)=0$, set $\hatP_h(s'|s,a,b)=\frac{1}{|\calS|}$ and $\hatR_h(s,a,b)=0$).  \\   
\nl           \scalebox{1}{$V_{H+1}^2(\cdot)\leftarrow 0$}.\\
\nl           \For{$h=H, \ldots, 1$}{
\hspace*{-0.5pt}\nl\hspace*{0.5pt}     
\For{all $s,a, b$}{
\hspace*{-1pt}\nl\hspace*{0.5pt}   \label{line: update UCBVI Q}            $Q^2_{h}(s, a, b)\leftarrow \min\big\{ \hatR_h(s,a,b) + \E_{s'\sim \hatP_h(\cdot|s,a,b)}\left[V^2_{h+1}(s')\right] + \bns^2_\tau, \ \ Q^2_h(s,a,b) \big\}$ \\ 
\hspace*{-1pt}\nl\hspace*{0.5pt}\label{line: update UCBVI V}   
                            $\textstyle V^2_h(s)\leftarrow \max_{a,b} Q^2_h(s, a, b)$   \\
\hspace*{-1pt}\nl\hspace*{0.5pt}                 
                            $\text{where\ } \tau=n_h(s,a,b)$. 
               }
            } 
    }
\end{algorithm}

In \pref{alg: MDP alg}, U1 maintains optimistic value function estimators $V^1_h(s), Q^1_h(s,a)$, and U2 maintains $V^2_h(s), Q^2_h(s,a,b)$ for every $h=1, \ldots, H$. Their constructions are based on two standard UCB-based algorithms. Specifically, the constructions of $V^1_h(s)$ and $Q^1_h(s,a)$ (\pref{line: UCBQ Q update}-\pref{line: UCBQ V update}) are similar to those of UCB Q-learning \citep{UCB-H_2018}, with the bonus term $\bns^1_\tau$ enlarged by a factor of $\sqrt{SB}$. Like in the bandit setting from \pref{sec: bandit}, U1 faces a non-stationary environment, and the extra $\sqrt{B}$ factor is used to compensate for the regret of U2 in future steps.\footnote{The extra $\sqrt{S}$ factor arises from a technical difficulty, and we are unsure whether it is necessary.} 
On the other hand, the constructions of $V_h^2(s)$ and $Q_h^2(s,a,b)$ (\pref{line: update UCBVI Q}-\pref{line: update UCBVI V}) are similar to those of UCBVI \citep{UCBVI_2017}. In particular, $V_h^2(s)$ is obtained by jointly optimizing over the actions of U1 and U2 (\pref{line: update UCBVI V}), conforming to our previous discussions. 

Perhaps the most intriguing is why we use UCB-H for U1 but UCBVI for U2. From a high level, this is because UCB-H shrinks its confidence set of value functions at a slower rate, while UCBVI is faster, which fulfills our need that U1 has to explore more in the early stages, for U2 to have enough time to find his optimal policy. Recall that the value iteration performed by U2 is through $V_h^2(s)\leftarrow \max_{a,b} Q_h^2(s,a,b)$ (\pref{line: update UCBVI V}). By the optimism principle, ideally we would like the agents to take actions $(a_h, b_h)=\argmax_{a,b}Q^2_h(s_h,a,b)$ to facilitate exploration. However, since U2 cannot control the actions taken by U1, and there is no communication between U1 and U2, it is unclear whether the optimism principle on $Q^2_h$ can be successfully carried out (the best U2 can do is to take $b_h=\argmax_{b}Q^2_h(s_h,a_h, b)$ for some $a_h$ taken by U1, as done in \pref{line: U2 selection rule}). Our key finding is that if $Q^1_h(s,a)$ always upper bounds $\max_b Q^2_h(s,a,b)$, then the agents can still perform adequate joint exploration without explicit coordination. This key property can be shown straightforwardly if we use UCB-H for U1 and UCBVI for U2 (see the proof of \pref{lem: key lemma}). 

Below, we establish some lemmas to be used in the regret bound analysis. The detailed proofs are deferred to \pref{app: mdp proofs}. We first define new notation with the episode indices.
\begin{definition}
    Let $Q_{t,h}^1(\cdot, \cdot)$, 
    $Q_{t,h}^2(\cdot,\cdot,\cdot)$
    be the 
    $Q_{h}^1(\cdot, \cdot)$, 
    $Q_{h}^2(\cdot,\cdot,\cdot)$
    at the beginning of episode $t$ in \pref{alg: MDP alg}. Let $s_{t,h}, a_{t,h}, b_{t,h}, r_{t,h}$ be the $s_h, a_h, b_h, r_h$ within episode $t$ in \pref{alg: MDP alg}.  
\end{definition}

\pref{lem: optimism Q2} below shows the optimism of U2's $Q$-function estimator, and relates the cumulative sum of $Q^2_{t,h}(s_{t,h}, a_{t,h}, b_{t,h})-Q_{*,h}(s_{t,h}, a_{t,h}, b_{t,h})$ to that of $V^2_{t,h+1}(s_{t,h+1})-V_{*,h+1}(s_{t,h+1})$. The proof is standard and we provide it in \pref{app: mdp proofs} for completeness. 
\begin{lemma} \label{lem: optimism Q2}
With probability at least $1-\order(\delta)$, $Q^2_{t,h}(s,a,b) \geq  Q_{*,h}(s,a,b)$ for all $t, h, s, a, b$, and  
\begin{align*}
    \sum_{t=1}^T &\left(Q^2_{t,h}(s_{t,h},a_{t,h},b_{t,h}) -  Q_{*,h}(s_{t,h},a_{t,h},b_{t,h})\right) \\
    &\qquad \leq \sum_{t=1}^T  \left(V^2_{t,h+1}(s_{t,h+1}) - V_{*,h+1}(s_{t,h+1})\right) + \otil\left(\sqrt{H^2S^2ABT}\right), \quad \forall h.   
\end{align*}
\end{lemma}
We remark that it is possible to improve the bound in \pref{lem: optimism Q2} by a factor of $\sqrt{S}$ by using the more refined analysis in \cite{UCBVI_2017} and defining $\bns^2_\tau$ to be a $\sqrt{S}$-factor smaller. However, as we will see below, this improvement will not lead to a better final regret bound, so in \pref{lem: optimism Q2} we opt to use a simpler analysis with a looser bound. 

Next, we establish our key lemma, \pref{lem: key lemma}, which states that U1 has \emph{more optimism} than U2. In this lemma we have to make $\bns^1_\tau$ a $\sqrt{SB}$-factor larger than that in \cite{UCB-H_2018}. While the $\sqrt{B}$ factor is necessary for the same reasons as in the bandit case, it is unclear whether the $\sqrt{S}$ factor is necessary. We leave the improvement of this factor as a future direction. 
\begin{lemma}\label{lem: key lemma}
With probability at least $1-\order(\delta)$, $Q_{t,h}^1(s,a)\geq \max_b Q_{t,h}^2(s,a,b)$ for all $t,h,s,a,b$.  
\end{lemma}
Finally, in \pref{lem: optimism Q1}, we relate the cumulative sum of $Q^1_{t,h}(s_{t,h}, a_{t,h}) - Q^1_{*,h}(s_{t,h}, a_{t,h})$ to that of $V^1_{t,h+1}(s_{t,h+1}) - V^1_{*,h+1}(s_{t,h+1})$. The proof is similar to that of \cite{UCB-H_2018}, but the bound is a $\sqrt{SB}$-factor larger than theirs due to the use of a larger bonus $\bns^1_\tau$. 
\begin{lemma}\label{lem: optimism Q1}
With probability at least $1-\order(\delta)$, 
\begin{align*}
     &\sum_{t=1}^T \left(Q^1_{t,h}(s_{t,h}, a_{t,h}) - Q_{*,h}(s_{t,h}, a_{t,h})\right)
     \\
     &\qquad \leq \left(1+\frac{1}{H}\right)\sum_{t=1}^T \left(V^1_{t,h+1}(s_{t,h+1}) - V_{*,h+1}(s_{t,h+1})\right) + \otil\left(\sqrt{H^3S^2A BT} + HSA\right), \quad \forall h.   
\end{align*}
\end{lemma}
Thanks to the fact that $V^1_{t,h}(s_{t,h})=Q^1_{t,h}(s_{t,h}, a_{t,h})$, \pref{lem: optimism Q1} leads to the following simple corollary. Note that we do not have a similar corollary for \pref{lem: optimism Q2} because  $V_{t,h}^2(s_{t,h})\neq Q^2_{t,h}(s_{t,h}, a_{t,h}, b_{t,h})$ (as discussed earlier, $(a_{t,h}, b_{t,h})$ is not necessarily equal to  $\argmax_{a,b}Q^2_{t,h}(s_{t,h}, a, b)$).  The proof of the corollary is also in \pref{app: mdp proofs}. 
\begin{cor}\label{cor: important cor}
    $\sum_{t=1}^T \left(Q^1_{t,h}(s_{t,h}, a_{t,h}) - Q_{*,h}(s_{t,h}, a_{t,h})\right) = \otil\left(\sqrt{H^5S^2ABT} + H^2SA\right)$. 
\end{cor}
Finally, we are able to show our main theorem: 
\begin{theorem}
    With probability $1-\order(\delta)$, \pref{alg: MDP alg} guarantees $\Reg(T)=\otil\left(H^{3.5}S\sqrt{ABT} + H^3SA\right)$. 
\end{theorem}
\begin{proof}
    We perform regret decomposition as follows: 
    \begin{align}
        &\Reg(T) = \sum_{t=1}^T \left(V_{*,1}(s_{t,1}) - V^{\pi^1_t, \pi^2_t}_{ 1}(s_{t,1})\right) \nonumber  \\
        &= \sum_{t=1}^T \sum_{h=1}^H \sum_{s,a, b} \P\left[(s_{t,h}, a_{t,h}, b_{t,h})=(s,a,b)~\big|~s_{t,1}, \pi^1_t, \pi^2_t\right] \left(V_{*,h}(s) - Q_{*,h}(s,a,b)\right) \tag{by the performance difference lemma \citep{kakade2002approximately}}  \\
        &= \sum_{t=1}^T \sum_{h=1}^H  \left(V_{*,h}(s_{t,h}) - Q_{*,h}(s_{t,h},a_{t,h},b_{t,h})\right) + \otil\left(H\sqrt{HT}\right)   \tag{by \pref{lem: UCB concentration}}\\
        &= \sum_{t=1}^T \sum_{h=1}^H  \left(V_{*,h}(s_{t,h}) - Q_{*,h}(s_{t,h}, a_{t,h})\right) + \sum_{t=1}^T \sum_{h=1}^H  \left( Q_{*,h}(s_{t,h}, a_{t,h}) -  Q_{*,h}(s_{t,h}, a_{t,h}, b_{t,h})\right) + \otil\left(\sqrt{H^3T}\right).  \label{eq: regret decompose}
    \end{align}
    Note that 
    \begin{align*}
        \Reg_h^1 &\triangleq  \sum_{t=1}^T \left(V_{*,h}(s_{t,h}) - Q_{*,h}(s_{t,h}, a_{t,h})\right) \leq \sum_{t=1}^T \left(Q^1_{t,h}(s_{t,h}, a_{t,h}) - Q_{*,h}(s_{t,h}, a_{t,h})\right) \\
        &\leq \otil\left(\sqrt{H^5S^2ABT} + H^2SA\right), \tag{by \pref{cor: important cor}}
    \end{align*}
    where the first inequality is because \pref{lem: optimism Q2}, \pref{lem: key lemma}, and the way U1 chooses $a_{t,h}$ yield $V_{*,h}(s_{t,h})=\max_{a,b} Q_{*,h}(s_{t,h}, a, b)\leq \max_{a,b}  Q_{t,h}^2(s_{t,h}, a, b) \leq \max_a Q_{t,h}^1(s_{t,h}, a) = Q_{t,h}^1(s_{t,h}, a_{t,h})$. 
    %
    %
%
 %
    On the other hand, 
    \allowdisplaybreaks
    \begin{align*}
        \Reg_h^2 &\triangleq  
        \sum_{t=1}^T \left( Q_{*,h}(s_{t,h}, a_{t,h}) -  Q_{*,h}(s_{t,h}, a_{t,h}, b_{t,h})\right) \\
        &\leq \sum_{t=1}^T \left(Q^2_{t,h}(s_{t,h}, a_{t,h}, b_{t,h}) - Q_{*,h}(s_{t,h}, a_{t,h}, b_{t,h})\right)\\
        &\leq \sum_{t=1}^T \left(V^2_{t,h+1}(s_{t,h+1}) - V_{*,h+1}(s_{t,h+1})\right) + \otil\left(\sqrt{H^2S^2ABT} + H^2SA\right) \tag{by \pref{lem: optimism Q2}}\\
        &\leq  \sum_{t=1}^T \left(V^1_{t,h+1}(s_{t,h+1}) - V_{*,h+1}(s_{t,h+1})\right) + \otil\left(\sqrt{H^5S^2ABT} + H^2SA\right)  \tag{by \pref{lem: key lemma} and the definitions of $V^1_{t,h}(s)$ and $V^2_{t,h}(s)$} \\
        &\leq \sum_{t=1}^T \left(Q^1_{t,h+1}(s_{t,h+1}, a_{t,h+1}) - Q_{*,h+1}(s_{t,h+1}, a_{t,h+1})\right) + \otil\left(\sqrt{H^5S^2ABT} + H^2SA\right) \tag{by the way U1 chooses $a_{t,h+1}$}\\
        &\leq  \otil\left(\sqrt{H^5S^2ABT} + H^2SA\right), \tag{by \pref{cor: important cor}}
    \end{align*}
    where the first inequality follows from \pref{lem: optimism Q2} and the way U2 chooses actions, resulting in $Q_{*,h}(s_{t,h},a_{t,h})=\max_b Q_{*,h}(s_{t,h},a_{t,h}, b)\leq \max_b Q^2_{t,h}(s_{t,h},a_{t,h}, b)=Q^2_{t,h}(s_{t,h},a_{t,h}, b_{t,h})$. 
    
    Combining the bounds on $\Reg^1_h$ and $\Reg^2_h$ with \pref{eq: regret decompose} proves the theorem. 
\end{proof}
\section{Conclusion and Future Directions}\label{sec:con}
In this work, we exploit the hierarchical information structure in hierarchical bandits/MDPs, and propose efficient and near-optimal algorithms that require no coordination or communication between the agents. A key feature of our algorithms is that the leader, i.e., the less-informed upper level agent, performs single-agent-like near-optimal subroutines with specially designed bonuses without the need of knowing or tracking the learning procedure of the lower level agent(s). One future direction is to explore other information structures that may allow simplification when more than two agents and reward information asymmetry come in. Extending the results to bandit problems with structural relationships between the payoffs of the arms and to general-sum games are also interesting directions.

\bibliography{References}

\appendix
\section{Supplementary Details}\label{app: supp}

\subsection{More Related Work}\label{app: more-relate}
Most papers on MAMAB consider a set of agents pulling \emph{the same set of arms} simultaneously, and in most of them the agents coordinate through \emph{real-time communications} to collaboratively find the optimal policy, with a few exceptions \citep{Bistritz_MAMAB1_2021,Bistritz_MAMAB2_2021, bubeck2021cooperative}. In the former, the communication resource is either costly \citep{Madhushani_MAMAB_2020}, limited by budget \citep{Lalitha_MAMAB_2021,Vial_RobustMAMAB_2021,Sankararaman_SocialMAMAB_2019,Chawla_Gossip_2020}, or constrained through communication networks \citep{Landgren_MAMAB_2020,Shahrampour_MAMAB_2017}, so the main focus is on designing communication efficient schemes that achieve the same performance as if there were no information asymmetry. In another related thread, referred to as the matching bandits problem, agents choosing the same arm collide and obtain zero rewards \citep{Kalathil_Matching_2014,Bistritz_MAMAB1_2021,Bistritz_MAMAB2_2021}; here and in a few other works \citep{Shahrampour_MAMAB_2017}, different agents get different distribution of rewards from the same arm, while in other referenced work they get independent and identically distributed samples from the same arm.

Regret minimization in MARL is in general challenging due to the fact that every agent faces a non-stationary environment. It has been shown in \citep{abbasi2013online, radanovic2019learning, tian2021online} that for single-agent non-stationary MDP problems, to have a sub-linear-in-$T$ regret bound against the best policy is both computationally and statistically hard. Therefore, to establish meaningful guarantees in MARL while keeping the algorithm efficient, special properties of the game have to be considered. \cite{radanovic2019learning} consider the same two-agent collaborative setting as ours, but requires that the agents exchange their policies after each episode. \cite{tian2021online} study another two-agent setting where each agent is agnostic about the actions of the other; however, their algorithm is conservative (with the goal of guarding against an adversarial opponent) and does not exploit the cooperative setting of our problem. \cite{leonardos2021global} study multi-agent Markov potential games (more general than the team problem) and establishes finite convergence bounds; however, their algorithm does not handle the state-space exploration issue (which is a key element in our work) so their regret bound has an extra problem-dependent factor; besides, only convergence to local Nash Equilibria is shown, and there is no guarantee about attaining global optima.

\subsection{The CI Approach}\label{app: CI-approach}
In two-agent cases, there are three obvious types of information structure in terms of \emph{action information asymmetry}.
\begin{itemize}
    \item Sequential decision making: U1 first chooses $a_t\in[A]$; after observing $a_t$, U2 then chooses $b_t\in[B]$. This is the hierarchical information structure considered in this paper.
    \item Simultaneous decision making: U1 and U2 choose $a_t\in[A]$ and $b_t\in[B]$ simultaneously, respectively. Depending on their respective feedback afterwards, the setting can be further divided as follows:
    \begin{itemize}
        \item Complete information sharing: both agents observe $(a_t,b_t)$ directly after they make their choices.
        \item No information sharing: both agents do not observe $(a_t,b_t)$. This setting is considered in \cite{Chang_MAMAB_2021}.
    \end{itemize}
\end{itemize}

In the setting of sequential decision making, U1 also does not observe $b_t$. Otherwise, it will be identical to the setting of complete information sharing in the learning context since after the time-step ends both agents will know $(a_t,b_t)$.

In the complete information sharing setting, it is evident that the agents may treat the joint action space $[A]\times[B]$ as the new action space and learns as if a single agent (the fictitious coordinator) is learning the policy of choosing the joint actions. The learning is centralized as there is no information asymmetry and the non-stationarity issue will not happen. Interestingly, the same approach can be carried through in the other two information structures as well. Suppose U1 learns with algorithm $\alg^1$ (which should also include any possible tie-breaking rule) and randomization seed $\Rf^1$, and U2 learns with $\alg^2$ and randomization seed $\Rf^2$. Suppose both agents have the information of $(\alg^1,\Rf^1,\alg^2,\Rf^2)$. In step $1$, U2 can generate $a_1$ from $(\alg^1,\Rf^1)$, and U1 can generate $b_1$, so that $(a_1,b_1)$ becomes CI. Going forward, in step $t$, since $\Ic_{t-1}=(a_{1:t-1},b_{1:t-1},r_{1:t-1})$ (where $a_{1:t-1}=(a_1,\dots,a_{t-1})$, etc.) is CI, U2 can reproduce $a_t$ from $\alg^1(\Ic_{t-1},\Rf^1)$, and U1 can reproduce $b_t$ from $\alg^2(\Ic_{t-1},\Rf^2)$, so that $(a_t,b_t)$ is again CI. We can see that with this approach there will be no information asymmetry. Clearly, if U1 and U2 treat $[A]\times[B]$ from the coordinator's perspective and adopt the same single-agent algorithm $\alg^1=\alg^2$ with near-optimal regret guarantee and the same randomization device $\Rf^1=\Rf^2$, the problem is equivalent to learning in the standard single-agent $AB$-armed bandit or standard single-agent MDP with action space being $[A]\times[B]$. Using the state of the art algorithms, i.e., UCB1 \citep{auer2002finite} for the bandit setting, and UCBVI algorithm (with a Bernstein bonus design) \citep{UCBVI_2017} (model-based) or the UCB-Advantage algorithm \citep{Zhang_UCB-A_2020} (model-free), one may achieve the lower regret bounds of $\otil(\sqrt{ABT})$ for the bandit setting and $\otil(\sqrt{H^2SABT})$ for the MDP setting.

Both agents knowing $(\alg^1,\Rf^1,\alg^2,\Rf^2)$ and being able to reproduce each other's computation is a strong assumption. In the case of hierarchical information structure, simpler and more efficient alternatives presented in this paper are possible.
\section{Omitted Proofs} \label{app: mdp proofs}

\begin{proof}\textbf{of \pref{thm: gap dependent bound}\ }
    We first bound the number of times a sub-optimal arm $a\in\calA^{\times}$ can be drawn by U1. Notice that with probability at least $1-\order(\delta)$, for any $t, a\in\calA^{\times}$, 
    \begin{align}
        &\muh_t(a) + c\sqrt{\frac{\log (T/\delta)}{n_t(a)^+}} + \sqrt{\frac{\kappa B\log(T/\delta)}{n_t(a)^+}}  \nonumber \\
        &\leq \frac{1}{n_t(a)^+}\sum_{\tau=1}^{t-1}\mu_{a, b_\tau} + 2c\sqrt{\frac{\log(T/\delta)}{n_t(a)^+}} + \sqrt{\frac{\kappa B\log(T/\delta)}{n_t(a)^+}}   \tag{by \pref{lem: UCB concentration}} \\
        &\leq \mu_{a,1} + 2c\sqrt{\frac{\log(T/\delta)}{n_t(a)^+}} + \sqrt{\frac{\kappa B\log(T/\delta)}{n_t(a)^+}}. \label{eq: gap direction one}
    \end{align}
    If $n_t(a) > \frac{(16\kappa B + 64c^2)\log(T/\delta)}{(\mu_{1,1}-\mu_{a,1})^2}$, then the last expression can further be upper bounded by $\mu_{a,1} + \frac{(\mu_{1,1}-\mu_{a,1})}{4} + \frac{(\mu_{1,1}-\mu_{a,1})}{4} < \mu_{1,1}$. 
    
    On the other hand, by \pref{ass: U2's algorithm gap} (which implies \pref{ass: U2's algorithm}), with high probability we have 
    \begin{align*}
        &\muh_t(1) + c\sqrt{\frac{\log(T/\delta)}{n_t(1)^+}} + \sqrt{\frac{\kappa B\log(T/\delta)}{n_t(1)^+}} \nonumber  \\
        &\geq \frac{1}{n_t(1)^+}\sum_{\tau=1}^{t-1} \mu_{1,b_\tau} + \sqrt{\frac{\kappa B\log(T/\delta)}{n_t(1)^+}}   \tag{by \pref{lem: UCB concentration}} \\
        &\geq \mu_{1,1}.   \label{eq: gap direction one}
    \end{align*}
    Combining them, we see that if $n_t(a) > \frac{(16\kappa B + 64c^2)\log(T/\delta)}{(\mu_{1,1}-\mu_{a,1})^2}$, then 
    \begin{align*}
        \muh_t(1) + c\sqrt{\frac{\log(T/\delta)}{n_t(1)^+}} + \sqrt{\frac{\kappa B\log(T/\delta)}{n_t(1)^+}} > \muh_t(a) + c\sqrt{\frac{\log(T/\delta)}{n_t(a)^+}} + \sqrt{\frac{\kappa B\log(T/\delta)}{n_t(a)^+}}. 
    \end{align*}
    By the way U1 selects arms, with high probability, she will not draw arm $a$ at round $t$. In other words, the number of draws for any $a\in\calA^{\times}$ is upper bounded by $\order\left(\frac{B\log(T/\delta)}{(\mu_{1,1}-\mu_{a,1})^2}\right)$. 
    
    Then we bound the overall regret. Define $n_t(a, b)\triangleq \sum_{\tau=1}^{t-1}\one[(a_\tau, b_\tau)=(a, b)]$. We have 
    \begin{align*}
\Reg(T)&=\sum_{(a,b)\in[A]\times[B]}(\mu_{1,1}-\mu_{a,b})\cdot n_{T+1}(a,b)\\
&=\sum_{(a,b)\in[A]\times[B]}(\mu_{1,1}-\mu_{a,1}+\mu_{a,1}-\mu_{a,b})\cdot n_{T+1}(a,b)\\
&=\sum_{a\in[A]}(\mu_{1,1}-\mu_{a,1})\cdot\sum_{b\in[B]}n_{T+1}(a,b)+\sum_{a\in[A]}\sum_{b\in[B]}(\mu_{a,1}-\mu_{a,b})\cdot n_{T+1}(a,b) \\
&\leq\sum_{a\in[A]}(\mu_{1,1}-\mu_{a,1})\cdot n_{T+1}(a)+ \sum_{a\in\calA^{\circ}}\sum_{b\in\calB_a^{\times}}  \order\left(\frac{ \log(T/\delta)}{\mu_{a,1}-\mu_{a,b}}\right) + \sum_{a\in\calA^{\times}}\order\left(\sqrt{Bn_{T+1}(a)\log(T/\delta)}\right), \tag{by \pref{ass: U2's algorithm gap}}\\
&\leq \sum_{a\in\calA^{\times}}\order\left(\frac{B\log(T/\delta)}{\mu_{1,1}-\mu_{a,1}}\right)+\sum_{a\in\calA^{\circ}}\sum_{b\in\calB_a^{\times}} \order\left( \frac{ \log(T/\delta)}{\mu_{a,1}-\mu_{a,b}}\right). \tag{because $n_{T+1}(a)=\order\left(\frac{B\log(T/\delta)}{(\mu_{1,1}-\mu_{a,1})^2}\right) $ for $a\in\calA^{\times}$}
\end{align*}
\end{proof}

    

\begin{proof}\textbf{of \pref{lem: optimism Q2}\ }
    First, we use induction to prove the following inequalities: 
    \begin{align}
        0\leq Q^2_{t,h}(s,a,b) - Q_{*,h}(s,a,b) \leq \E_{s'\sim P(\cdot|s,a,b)}\left[V^2_{t,h+1}(s') - V_{*,h+1}(s')\right] + 2\bns^2_\tau, \label{eq: induction hypo UCBVI}
    \end{align}
    where $\tau=n_{t,h}(s,a,b)$, for all $s,a,b$. The order of induction is from $t=1$ to $t=T$, and (within each $t$) from $h=H$ to $h=1$. 
    
    For $t=1$, we have $Q^2_{1,h}(s,a,b) - Q_{*,h}(s,a,b)=H-Q_{*,h}(s,a,b)\geq 0$ and that $Q^2_{1,h}(s,a,b) - Q_{*,h}(s,a,b)\leq H\leq 2\bns^2_0$. 
    Suppose that the inequality holds for all $(t',h')$ with either $t'<t$, or $t'=t$ and $h'>h$. Fix a $(s,a,b)$ and let $\tau=n_{t,h}(s,a,b)$. 
    By the update rule of $Q^2_{t,h}(s,a,b)$, we have $Q^2_{t,h}(s,a,b) = \min\big\{\hat{Q}^2_{t,h}(s,a,b), Q^2_{t-1,h}(s,a,b)\big\}$ where 
    \begin{align*}
        \hat{Q}^2_{t,h}(s, a, b)
        &= \hatR_{t,h}(s,a,b) + \E_{s'\sim \hatP_{t,h}(\cdot|s,a,b)}\left[V^2_{t,h+1}(s')\right] + \bns^2_\tau. \tag{$\tau=n_{t,h}(s,a,b)$}
    \end{align*}
    Besides, 
    \begin{align*}
        Q_{*,h}(s,a,b) = R(s,a,b) + \E_{s'\sim P(\cdot|s,a,b)}\left[V_{*,h}(s')\right]. 
    \end{align*}
    Taking their difference, we get 
    \begin{align}
        &\hat{Q}^2_{t,h}(s, a, b) - Q_{*,h}(s,a,b)   \nonumber  \\
        &= \left(\hatR_{t,h}(s,a,b) - R(s,a,b)\right) + \underbrace{\E_{s'\sim P(\cdot|s,a,b)}\left[V^2_{t,h+1}(s') - V_{*,h+1}(s')\right]}_{\term_1}   \nonumber \\
        &\qquad \quad + \underbrace{\left(\E_{s'\sim \hatP_{t,h}(\cdot|s,a,b)}\left[V^2_{t,h+1}(s')\right] - \E_{s'\sim P(\cdot|s,a,b)}\left[V^2_{t,h+1}(s')\right]\right)}_{\term_2}  + \bns^2_\tau.  \label{eq: recursive temp}
    \end{align}
    By \pref{lem: UCB concentration} and \pref{lem: Weissman}, for some universal constant $c>0$,
    \begin{align}
        &\left|\hatR_{t,h}(s,a,b) - R(s,a,b)\right| \leq  \frac{1}{2}c\sqrt{\frac{\log(T/\delta)}{\tau^+}},   \label{eq: to be used later 1}\\
        &\left| \term_2 \right| \leq \left\|\hatP_{t,h}(\cdot|s,a,b) - P(\cdot|s,a,b)\right\|_1 \|V^2_{t,h+1}\|_\infty \leq \frac{1}{2} cH\sqrt{\frac{S\log(T/\delta)}{\tau^+}}, \label{eq: to be used later 2} 
    \end{align}
    and therefore $ \left|\hatR_{t,h}(s,a,b) - R(s,a,b)\right| + |\term_2| \leq \bns^2_\tau$\,. Combining this with \pref{eq: recursive temp}, we get 
    \begin{align}
        \term_1 \leq \hat{Q}^2_{t,h}(s, a, b) - Q_{*,h}(s,a,b) \leq \term_1 + 2\bns^2_\tau\,.  \label{eq: UCBVI temp}
    \end{align}
    Using $Q^2_{t,h}(s,a,b)\leq \hat{Q}^2_{t,h}(s,a,b)$, \pref{eq: UCBVI temp} implies the right inequality in \pref{eq: induction hypo UCBVI}. 
    
    To prove the left inequality in  \pref{eq: induction hypo UCBVI}, notice that if $h=H$, then $\term_1=0$; if $h<H$, 
    \begin{align*}
        \term_1 &\geq \min_{s'} V^2_{t,h+1}(s') - V_{*,h+1}(s')\\
        &=\min_{s'}\left(\min\left\{\max_{a',b'}Q^2_{t,h+1}(s',a',b'), H\right\} - \max_{a', b'}Q_{*,h+1}(s', a', b')\right)\\ &=\min_{s'}\left(\min\left\{\max_{a',b'}Q^2_{t,h+1}(s',a',b')- \max_{a', b'}Q_{*,h+1}(s', a', b'), H- \max_{a', b'}Q_{*,h+1}(s', a', b')\right\} \right)\\
        &\geq 0. 
    \end{align*}
    where the last inequality is by the induction hypothesis. 
    
    Thus, $\term_1\geq 0$. Together with \pref{eq: UCBVI temp}, we get $\hat{Q}^2_{t,h}(s, a, b) - Q_{*,h}(s,a,b)\geq 0$. By the induction hypothesis, we also have $Q^2_{t-1,h}(s,a,b)\geq Q_{*,h}(s,a,b)$. 
    Therefore,  \sloppy$Q^2_{t,h}(s,a,b) = \min\big\{\hat{Q}^2_{t,h}(s,a,b), Q^2_{t-1,h}(s,a,b)\big\}\geq Q_{*,h}(s,a,b)$, proving the left inequality in \pref{eq: induction hypo UCBVI}.

    Based on \pref{eq: induction hypo UCBVI}, we can write
    \begin{align*}
        &\sum_{t=1}^T \left(Q^2_{t,h}(s_{t,h}, a_{t,h}, b_{t,h})-Q_{*,h}(s_{t,h}, a_{t,h}, b_{t,h})\right) \\ &\leq \sum_{t=1}^T \E_{s'\sim P(\cdot|s_{t,h},a_{t,h},b_{t,h})}\left[V^2_{t,h+1}(s') - V_{*,h+1}(s')\right] + \sum_{t=1}^T 2\bns^2_{\tau_{t,h}}   \tag{define $\tau_{t,h}=n_{t,h}(s_{t,h}, a_{t,h}, b_{t,h})$} \\
        &\triangleq \sum_{t=1}^T \left(V^2_{t,h+1}(s_{t,h+1}) - V_{*,h+1}(s_{t,h+1}) + \varepsilon_{t,h}\right) + \sum_{t=1}^T 2\bns^2_{\tau_{t,h}} \tag{define $\varepsilon_{t,h}$ to be the difference}
    \end{align*}
    Since $\varepsilon_{t,h}$ is zero-mean, by \pref{lem: UCB concentration}, 
    \begin{align*}
        \sum_{t=1}^T \varepsilon_{t,h} = \order\left(H\sqrt{T\log(T/\delta)}\right). 
    \end{align*}
    Besides, 
    \begin{align*}
        \sum_{t=1}^T \bns^2_{\tau_{t,h}} = \order\left(\sum_{t=1}^T
 H\sqrt{\frac{S\log(T/\delta)}{n_{t,h}(s_{t,h}, a_{t,h}, b_{t,h})^+}} \right) = \order\left(HS\sqrt{ABT\log(T/\delta)}\right).    
    \end{align*}
 Combining the three bounds above proves the second conclusion in the lemma.     
\end{proof}

\begin{proof}\textbf{of \pref{lem: key lemma}\ }
    We use induction to show the desired inequality. Again, the order of induction is from $t=1$ to $t=T$, and (within each $t$) from $h=H$ to $h=1$. When $t=1$, $Q^1_{1,h}(s,a)=H=\max_b Q^2_{1,h}(s,a,b)$. 
    
    Suppose that the inequality holds for all $(t',h')$ with $t'<t$, or $t'=t$ and $h'>h$. 
    Let $\tau=n_{t,h}(s,a)$, and let $1\leq t_1 < t_2 < \cdots < t_\tau < t$ be the episodes in which $(s,a)$ is visited at step $h$. By the update rule of $Q^1_{t,h}(\cdot,\cdot)$, we have 
    \allowdisplaybreaks
    \begin{align*}
        &Q^1_{t,h}(s,a)\\ 
        &= \alpha^0_\tau H + \sum_{i=1}^\tau \alpha^i_\tau \left( \ri + V^1_{t_i, h+1}(\si) + \bns^1_i \right) \tag{define $\alpha^i_\tau=\alpha_i \Pi_{j=i+1}^\tau (1-\alpha_j)$ for $1\leq i\leq \tau$ and $\alpha^0_\tau=\Pi_{j=1}^\tau (1-\alpha_j)$}\\
        &\geq \alpha^0_\tau H + \sum_{i=1}^\tau \alpha^i_\tau \left(R(s, a, \bi) + \sum_{s'} P(\cdot|s,a,\bi) V^1_{t_i, h+1}(s')\right) + \frac{1}{2}\bns_\tau^1     \tag{see the explanation below indexed $\star$}\\
        &\geq \alpha^0_\tau H + \sum_{i=1}^\tau \alpha^i_\tau \left(R(s, a, \bi) + \sum_{s'} P(\cdot|s,a,\bi) V^2_{t_i, h+1}(s')\right) + \frac{1}{2}\bns_\tau^1  \tag{by the induction hypothesis}  \\
        &\geq \alpha^0_\tau H + \sum_{i=1}^\tau \alpha^i_\tau \left(\hatR_{t_i,h}(s, a, \bi) + \sum_{s'}\hatP_{t_i,h}(\cdot|s,a,\bi) V^2_{t_i, h+1}(s') - \bns^2_{\xi_i}\right) + \frac{1}{2}\bns_\tau^1  \tag{define $\xi_i=n_{t_i, h}(s,a,\bi)$ and use \pref{eq: to be used later 1} and \pref{eq: to be used later 2}} \\
        &\geq \alpha^0_\tau H + \sum_{i=1}^\tau \alpha^i_\tau Q^2_{t_i,h}(s,a,\bi) - 2\sum_{i=1}^\tau \alpha^i_\tau \bns^2_{\xi_i} + \frac{1}{2}\bns_\tau^1    \tag{by the definition of $Q^2_{t,h}(s,a,b)$}\\
        &\geq \alpha^0_\tau H + \sum_{i=1}^\tau \alpha^i_\tau \max_b Q^2_{t_i,h}(s,a,b) \tag{see the explanation below indexed $\star\star$} \\
        &\geq \max_b Q^2_{t,h}(s,a,b)  \tag{because $Q^2_{t,h}$ is non-increasing in $t$ and $\sum_{i=0}^\tau \alpha^i_\tau=1$} 
    \end{align*}
    In the first inequality ($\star$), we use the fact that $\sum_{i=1}^\tau \alpha^i_\tau \bns^1_i\geq \bns^1_\tau$ by the first item in \pref{lem: lemmas regarding alpha}, and that 
    \begin{align}
        &\left|\sum_{i=1}^\tau \alpha^i_\tau \left(R(s,a,b_{t_i,h}) + \sum_{s'}P(s'|s,a,b_{t_i,h})V^1_{t_i, h+1}(s') - r_{t_i,h} - V^1_{t_i, h+1}(s_{t_i, h+1})\right)\right| \nonumber \\
        &\leq \frac{1}{2}c'H\sqrt{\frac{HS\log(T/\delta)}{\tau^+}} \leq \frac{1}{2}\bns^1_\tau \label{eq: U1 concentration}
    \end{align}
    for some universal constant $c'>0$ by \pref{lem: alpha concentration}. 
    
    In the penultimate inequality ($\star\star$), we first use the selection rule of $b_{t,h}=\argmax_{b}Q^2_{t,h}(s_{t,h},a_{t,h},b)$, and then use the following \pref{lem: seq lemma} to bound 
    \begin{align*}
        2\sum_{i=1}^\tau \alpha^i_\tau \bns^2_{\xi_i} &= 2cH\sqrt{S\log(T/\delta)}\sum_{i=1}^\tau \alpha^i_\tau \frac{1}{\sqrt{n_{t_i,h}(s,a,b_{t_i,h})^+}} \\
        &\leq 2cH\sqrt{S\log(T/\delta)} \times 4\sqrt{\frac{BH}{n_{t,h}(s,a)^+}} \\
        &\leq \frac{1}{2}c'H\sqrt{\frac{HSB\log(T/\delta)}{n_{t,h}(s,a)^+}} = \frac{1}{2}\bns^1_\tau.  
    \end{align*}
\end{proof}

\begin{lemma}\label{lem: seq lemma}
    Let $\{\tau_1, \ldots, \tau_B\}$ be non-negative integers such that $\sum_{b=1}^B \tau_b=\tau$. 
    Define for all $b=1, \ldots, B$: 
    \begin{align*}
        Y_{bi} = \frac{1}{\sqrt{(i-1)^+}}, \qquad \text{for\ } i = 1, 2, \ldots, \tau_b.  
    \end{align*}
    Let $\{Z_1, Z_2, \ldots, Z_\tau\}$ be any permutation of
    \begin{align*}
        \left\{Y_{11}, Y_{12}, \ldots, Y_{1\tau_1}, Y_{21}, Y_{22}, \ldots, Y_{2\tau_2}, \ldots \ldots Y_{B1}, Y_{B2}, \ldots, Y_{B\tau_B}\right\} 
    \end{align*}
    Then 
    \begin{align*}
        \sum_{i=1}^\tau \alpha^i_\tau Z_i \leq 4\sqrt{\frac{BH}{\tau^+}}. 
    \end{align*}
\end{lemma}
\begin{proof}
    We write $i=\phi(b,j)$ if $Y_{bj}$ is mapped to $Z_i$. Also, define $\Phi(b) = \{\phi(b,j):~ j\in[\tau_b]\}$ as the set of indices in $\{Z_i\}$ that are mapped from $\{Y_{b1}, \ldots, Y_{b\tau_b}\}$. 

    We first show the following claim: for all $b$, 
    \begin{align}
        \sum_{i\in\Phi(b)}\alpha^i_\tau Z_i \leq 2\sqrt{2\alpha_\tau \sum_{i\in\Phi(b)}\alpha^i_\tau}.   \label{eq: to show}
    \end{align}
    To show \pref{eq: to show}, observe that the left-hand side is equal to 
    \begin{align}
        \sum_{i\in\Phi(b)}\alpha^i_\tau Z_i = \sum_{j=1}^{\tau_b}  \alpha_\tau^{\phi(b,j)} Y_{bj} = \sum_{j=1}^{\tau_b}  \alpha_\tau^{\phi(b,j)} \frac{1}{\sqrt{(j-1)^+}}\leq \sum_{j=1}^{\tau_b}  \alpha_\tau^{\phi(b,j)} \sqrt{\frac{2}{j}}. \label{eq: our solution}
    \end{align}
    By the definition of $\alpha^i_\tau$, we have $\alpha^i_\tau\leq \alpha_\tau$ for any $i$. We see that the last expression in 
    \pref{eq: our solution} is upper bounded by the optimal solution of the following programming: 
    \begin{align*}
        &\max_{\beta_j} \sum_{j=1}^{\tau_b} \beta_j \sqrt{\frac{2}{j}} \\
        \text{s.t.}
        &~~ \sum_{j=1}^{\tau_b}\beta_j \leq \sum_{i\in\Phi(b)}\alpha^i_\tau \\
        &\quad 0\leq \beta_j \leq \alpha_\tau\ \  \forall j
    \end{align*}
    This programming exhibits a greedy solution that sets $\beta_j=\alpha_\tau$ for $j\leq j^\star\triangleq \left\lfloor \frac{1}{\alpha_\tau} \sum_{i\in\Phi(b)}\alpha^i_\tau\right\rfloor$, $\beta_j=\sum_{i\in\Phi(b)}\alpha^i_\tau - \alpha_\tau j^\star$ for $j=j^\star+1$, and $\beta_j=0$ otherwise. The optimal value of this solution is upper bounded by 
    \begin{align*}
        \alpha_\tau \sum_{j=1}^{j^\star}  \sqrt{\frac{2}{j}} + \left(\sum_{i\in\Phi(b)}\alpha^i_\tau - \alpha_\tau j^\star\right)\sqrt{\frac{2}{j^\star+1}} \leq \alpha_\tau\int_{0}^{\frac{1}{\alpha_\tau}\sum_{i\in\Phi(b)}\alpha^i_\tau} \sqrt{\frac{2}{x}}\mathrm{d}x = 2 \sqrt{2\alpha_\tau\sum_{i\in\Phi(b)}\alpha^i_\tau}, 
    \end{align*}
    showing \pref{eq: to show}. To get the final bound, we sum this bound over $b$ and use the definition of $\alpha_\tau$: 
    \begin{align*}
        \sum_{i=1}^\tau\alpha^i_\tau Z_i = \sum_{b=1}^B \sum_{i\in\Phi(b)}\alpha^i_\tau Z_i \leq \sum_{b=1}^B  2 \sqrt{2\alpha_\tau\sum_{i\in\Phi(b)}\alpha^i_\tau} \leq 2 \sqrt{2B\alpha_\tau\sum_{i=1}^\tau \alpha^i_\tau} = 2\sqrt{\frac{2B(H+1)}{H+\tau}}\leq 4\sqrt{\frac{BH}{\tau^+}}, 
    \end{align*}
    where in the second inequality we use the AM-GM inequality and in the last equality we use $\sum_{i=1}^\tau\alpha^i_\tau=1$ for $\tau\geq 1$. 
\end{proof}

\begin{proof}\textbf{of \pref{lem: optimism Q1}\ }
    
    Fix $t, h, s, a$. Let $\tau=n_{t,h}(s,a)$, and let $1\leq t_1 < t_2 < \cdots < t_\tau < t$ be the episodes in which $(s,a)$ is visited at layer $h$. By the update rule of $Q^1_{t,h}(\cdot,\cdot)$, we have 
    \allowdisplaybreaks
    \begin{align*}
        &Q^1_{t,h}(s,a)\\ 
        &= \alpha^0_\tau H + \sum_{i=1}^\tau \alpha^i_\tau \left( r_{t_i, h} + V^1_{t_i, h+1}(s_{t_i, h+1}) + \bns^1_i \right) \\
        &\leq \alpha^0_\tau H + \sum_{i=1}^\tau \alpha^i_\tau \left(R(s, a, b_{t_i, h}) + \E_{s'\sim P(\cdot|s,a,b_{t_i, h})}\left[V^1_{t_i, h+1}(s')\right]\right) + \order(\bns_\tau^1) \tag{by \pref{lem: alpha concentration} and that $\sum_{i=1}^\tau \alpha^i_\tau \bns^1_i\leq 2\bns^1_\tau$ by the first item in \pref{lem: lemmas regarding alpha}} \\
        &= \alpha^0_\tau H + \sum_{i=1}^\tau \alpha^i_\tau \left(R(s, a, b_{t_i, h}) + \E_{s'\sim P(\cdot|s,a,b_{t_i, h})}\left[V_{*, h+1}(s')\right]\right) \\
        &\qquad + \sum_{i=1}^\tau \alpha^i_\tau \left(\E_{s'\sim P(\cdot|s,a,b_{t_i,h})}\left[V^1_{t_i, h+1}(s') - V_{*,h+1}(s')\right]\right)  + \order(\bns_\tau^1) \\
        &= \alpha^0_\tau H + \sum_{i=1}^\tau \alpha^i_\tau Q_{*,h}(s,a,b_{t_i,h})+ \sum_{i=1}^\tau \alpha^i_\tau \left(\E_{s'\sim P(\cdot|s,a,b_{t_i,h})}\left[V^1_{t_i, h+1}(s') - V_{*,h+1}(s')\right]\right)  + \order(\bns_\tau^1) \\
        &\leq \alpha^0_\tau H + \sum_{i=1}^\tau \alpha^i_\tau Q_{*,h}(s,a,b_{t_i,h})+ \sum_{i=1}^\tau \alpha^i_\tau \left(V^1_{t_i, h+1}(s_{t_i, h+1}) - V_{*,h+1}(s_{t_i, h+1})\right)  + \order(\bns_\tau^1)  \tag{by \pref{lem: alpha concentration}}
    \end{align*}
    Therefore,  
    \begin{align}
        &Q^1_{t,h}(s,a) - Q_{*,h}(s,a)   \nonumber \\
        &= \alpha^0_\tau(H-Q_{*,h}(s,a)) + \sum_{i=1}^\tau \alpha^i_\tau \left(Q_{*,h}(s,a,b_{t_i,h}) - Q_{*,h}(s,a)\right) \nonumber  \\
        & \qquad + \sum_{i=1}^\tau \alpha^i_\tau \left(V^1_{t_i, h+1}(s_{t_i,h+1}) - V_{*,h+1}(s_{t_i,h+1})\right)   + \order(\bns_\tau^1) \nonumber    \\
        &\leq \alpha^0_\tau H + \sum_{i=1}^\tau \alpha^i_\tau \left(V^1_{t_i, h+1}(s_{t_i,h+1}) - V_{*,h+1}(s_{t_i,h+1})\right)  + \order(\bns_\tau^1)   \label{eq: tmpotary}
    \end{align}
    Now consider the cumulative sum, and define $t_{i}(s,a)$ to be the index of the episode when it is the $i$-th time $(s,a)$ is visited at layer $h$. 
    \allowdisplaybreaks
    \begin{align*}
        &\sum_{t=1}^T  \left(Q^1_{t,h}(s_{t,h}, a_{t,h}) - Q_{*,h}(s_{t,h}, a_{t,h})\right) \\
        &= \sum_{s,a}\sum_{i=1}^{n_{T+1}(s,a)}\left(Q^1_{t_i(s,a),h}(s_{t_i(s,a),h}, a_{t_i(s,a),h}) - Q_{*,h}(s_{t_i(s,a),h}, a_{t_i(s,a),h})\right) \\
        &\leq \sum_{s,a}\sum_{i=1}^{n_{T+1}(s,a)} \left(\alpha^0_i H + \sum_{j=1}^{i-1}\alpha^j_i\left(V^1_{t_j(s,a), h+1}(s_{t_j(s,a),h+1}) - V_{*,h+1}(s_{t_j(s,a),h+1})\right) + \bns^1_{i-1}\right) \tag{by \pref{eq: tmpotary}}\\
        &= \sum_{s,a}\sum_{i=1}^{n_{T+1}(s,a)}\alpha^0_i H + \sum_{s,a}\sum_{j=1}^{n_{T+1}(s,a)-1} \sum_{i=j+1}^{n_{T+1}(s,a)}\alpha^j_i \left(V^1_{t_j(s,a), h+1}(s_{t_j(s,a),h+1}) - V_{*,h+1}(s_{t_j(s,a),h+1})\right) \\
        &\qquad \qquad + \sum_{s,a}\sum_{i=1}^{n_{T+1}(s,a)} \order\left(\sqrt{\frac{H^3SB\log(T/\delta)}{\max\{i-1,1\}}}\right) \\
        &\leq HSA + \sum_{s,a}\sum_{j=1}^{n_{T+1}(s,a)-1} \left(1+\frac{1}{H}\right) \left(V^1_{t_j(s,a), h+1}(s_{t_j(s,a),h+1}) - V_{*,h+1}(s_{t_j(s,a),h+1})\right) + \order\left( \sqrt{H^3S^2ABT} \right)  \tag{by the third item of \pref{lem: lemmas regarding alpha}} \\
        &\leq  \left(1+\frac{1}{H}\right)\sum_{t=1}^T \left(V^1_{t,h+1}(s_{t,h+1}) - V^1_{*,h+1}(s_{t,h+1} ) \right) + \order\left(\sqrt{H^3S^2A BT} + HSA\right).   
    \end{align*}
\end{proof}

\begin{proof}{\textbf{of \pref{cor: important cor}}\ }
By \pref{lem: optimism Q1} and the fact that  $V^1_{t,h}(s_{t,h})=Q^1_{t,h}(s_{t,h}, a_{t,h})$, we have 
\begin{align*}
     &\sum_{t=1}^T \left(Q^1_{t,h}(s_{t,h}, a_{t,h}) - Q_{*,h}(s_{t,h}, a_{t,h})\right)
     \\
     &\qquad \leq \left(1+\frac{1}{H}\right)\sum_{t=1}^T \left(Q^1_{t,h+1}(s_{t,h+1}, a_{t,h+1}) - Q_{*,h+1}(s_{t,h+1}, a_{t,h+1})\right) + \order\left(\sqrt{H^3S^2A BT} + HSA\right),    
\end{align*}
which gives 
\begin{align*}
    \sum_{t=1}^T \left(Q^1_{t,h}(s_{t,h}, a_{t,h}) - Q_{*,h}(s_{t,h}, a_{t,h})\right) 
    &\leq H\times \left(1+\frac{1}{H}\right)^H \times \order\left(\sqrt{H^3S^2A BT} + HSA\right) \\
    &= \order\left(\sqrt{H^5S^2A BT} + H^2SA\right)
\end{align*}
by expanding the recursion. 
\end{proof}


\section{Multiple Follower Extension} \label{app: multiple follower}
In this section, we consider the multiple follower case. Let $A$ be the number of actions of the leader, and let $B^i$ be the number of actions of the $i$-th follower, for $i=1, 2, \ldots, N$. We define $B=\frac{1}{N}\sum_{i=1}^N B^i$ to be the average number of actions of all followers. In each round $t$, the leader first selects an action $a_t\in[A]$. Then based on the information of $a_t$, each follower $i$ selects an action $b_t^i\in[B^i]$. The reward that the $i$-th follower receives is $r_t^i$, whose mean is $\mu^i_{a_t, b_t^i}$; the reward the leader receives is $r_t=\frac{1}{N}\sum_{i=1}^N r_t^i$, the average-reward over $i$. \footnote{Similar to the single-agent case, our framework also handles the case where the leader observes another fresh sample with mean $\frac{1}{N}\sum_{i=1}^N \mu^i_{a_t, b_t^i}$\,. } 

Our algorithm is presented in \pref{alg:4.2}. On the follower side, the algorithm is identical to the single-agent case (\pref{alg: two-stage alg}) -- for each arm of the leader, a simple empirical mean is maintained as in \pref{eq: algorithm 1 specification}. For the followers, similarly, we assume that all of them use a no-regret algorithm that satisfies the following assumption: 
\begin{ass}\label{ass: multiple follower ass}
    Every \follower~i guarantees the following
    for some universal constant $\kappa\geq 1$ with probability at least $1-\delta$: 
    \begin{align*}
        \forall t, a, \qquad \sum_{\tau=1}^{t} \I[a_\tau=a]\left(\max_{b^i}\mu^i_{a,b^i} - \mu^i_{a, b^i_\tau}\right) \leq \sqrt{\kappa B^i\sum_{\tau=1}^t \I[a_\tau=a] \log(T/\delta)}.  
    \end{align*}
\end{ass}
With \pref{ass: multiple follower ass}, the regret bound of \pref{alg:4.2} can be shown as in the following theorem:

\begin{theorem}
    With probability at least $1-\order(\delta)$, for all $a$ and $\{b^i\}_{i=1}^N$, 
    \begin{align*}
        \sum_{t=1}^T \left(\frac{1}{N}\sum_{i=1}^N \mu^i_{a, b^i} - \frac{1}{N}\sum_{i=1}^N \mu^i_{a_t, b_t^i}\right) = \order\left(\sqrt{ABT\log(T/\delta)}\right). 
    \end{align*}
\end{theorem}

\begin{algorithm}[t]
\caption{UCB for Hierarchical Bandits with Multiple Followers}\label{alg:4.2}
\textbf{define}: $c>0$ is a universal constant. \\
Followers start running algorithms satisfying \pref{ass: multiple follower ass} with some $\kappa\geq 1$ (we denote the instance of the algorithm by Follower $i$ under action $a\in[A]$ as $\alg^i(a)$). \\
\For{$t=1, 2, \dots,T$}{
\leader chooses $a_t\in\underset{a\in[A]}{\argmax}\ \muh_t(a)+c\sqrt{\frac{\log(T/\dl)}{n_t(a)^+}} + \sqrt{\frac{\kappa B\log(T/\delta)}{n_t(a)^+}}$. ($\muh_t(a)$, $n_t(a)$ defined in \pref{eq: algorithm 1 specification})\\
\For{$i=1, \ldots, N$ (in parallel)}{
$\follower~i$ observes $a_t$ and calls $\alg^i(a_t)$, which outputs an action $b_t^i$. \\ 
$\follower~i$ chooses $b_t^i$. \\
$\follower~i$ observes $r^i_t$ with $\E[r^i_t]=\mu^i_{a_t, b_t^i}$\,.
}
\leader observes $r_t= \frac{1}{N}\sum_{i=1}^N r_t^i$\,.  
}
\end{algorithm}

\begin{proof}
    \allowdisplaybreaks
    \begin{align*}
         &\sum_{t=1}^T \left(\frac{1}{N}\sum_{i=1}^N \mu^i_{a, b^i} - \frac{1}{N}\sum_{i=1}^N \mu^i_{a_t, b_t^i}\right) \\
         &\leq \sum_{t=1}^T \left(  \frac{1}{N}\sum_{i=1}^N \left(\frac{1}{n_t(a)^+}\times \sum_{\tau=1}^{t-1} \mu^{i}_{a, b^i_\tau} + \sqrt{\frac{\kappa B^i\log(T/\delta)}{n_t(a)^+}}\right) - \frac{1}{N}\sum_{i=1}^N \mu^i_{a_t, b_t^i}\right)  \tag{\pref{ass: multiple follower ass}}\\
         &\leq \sum_{t=1}^T \left(  \frac{1}{n_t(a)^+}\left( \frac{1}{N}\sum_{i=1}^N  \sum_{\tau=1}^{t-1} \mu^{i}_{a, b^i_\tau}\right) + \sqrt{\frac{\kappa B\log(T/\delta)}{n_t(a)^+}} - \frac{1}{N}\sum_{i=1}^N \mu^i_{a_t, b_t^i}\right)  \tag{Cauchy-Schwarz inequality} \\
         &\leq \sum_{t=1}^T \left( \muh_t(a) + c\sqrt{\frac{\log(T/\delta)}{n_t(a)^+}} + \sqrt{\frac{\kappa B\log(T/\delta)}{n_t(a)^+}} - \frac{1}{N}\sum_{i=1}^N \mu^i_{a_t, b_t^i}\right)  \tag{by \pref{lem: UCB concentration}} \\
         &\leq \sum_{t=1}^T \left( \muh_t(a_t) + c\sqrt{\frac{\log(T/\delta)}{n_t(a_t)^+}} + \sqrt{\frac{\kappa B\log(T/\delta)}{n_t(a_t)^+}} - \frac{1}{N}\sum_{i=1}^N \mu^i_{a_t, b_t^i}\right)  \tag{by the algorithm} \\
         &= \sum_{t=1}^T \left( \muh_t(a_t) - \frac{1}{N}\sum_{i=1}^N \mu^i_{a_t, b_t^i}\right) + \order\left(\sqrt{ABT\log(T/\delta)}\right)  \\
         &= \underbrace{\sum_{a\in[A]}\sum_{t=1}^T \I[a_t=a]\left(\muh_t(a) - \max_{\{b^i\}} \frac{1}{N}\sum_{i=1}^N \mu^i_{a, b^i} \right)}_{\term_1} \\
         &\qquad + \underbrace{\sum_{a\in[A]}\sum_{t=1}^T \I[a_t=a]\left(\max_{\{b^i\}} \frac{1}{N}\sum_{i=1}^N \left(\mu^i_{a, b^i} - \mu^i_{a, b_t^i}\right) \right)}_{\term_2} + \order\left(\sqrt{ABT\log(T/\delta)}\right) 
    \end{align*}
    Under \pref{ass: multiple follower ass}, we can upper bound
    $\term_2$ by
    \begin{align*}
        \order\left(\sum_{a\in[A]}\frac{1}{N}\sum_{i=1}^N\sqrt{B^i n_{T+1}(a)\log(T/\delta)}\right) 
        &= \order\left(\sum_{a\in[A]}\sqrt{Bn_{T+1}(a)\log(T/\delta)}\right)  \tag{Cauchy-Schwarz} \\
        &= \order\left(\sqrt{ABT\log(T/\delta)}\right). 
    \end{align*}
    Besides, for all $t$ and $a$, we have with probability at least $1-\order(\delta)$, 
    \begin{align*}
        \muh_t(a) &= \frac{1}{n_t(a)^+}\sum_{\tau=1}^{t-1} \I[a_\tau=a]r_\tau \\
        &\leq \frac{1}{n_t(a)^+}\sum_{\tau=1}^{t-1} \I[a_\tau=a]\left(\frac{1}{N}\sum_{i=1}^N  \mu_{a, b^i_\tau} \right) + \order\left(\sqrt{\frac{\log(T/\delta)}{n_t(a)^+}}\right)  \tag{Azuma's inequality}  \\
        &\leq \frac{1}{N}\max_{\{b^i\}} \sum_{i=1}^N \mu_{a, b^i}  + \order\left(\sqrt{\frac{\log(T/\delta)}{n_t(a)^+}}\right).  
    \end{align*}
    Therefore, 
    \begin{align*}
        \term_1 \leq \order\left(\sum_{a\in[A]}\sum_{t=1}^T \I[a_t=a] \sqrt{\frac{\log(T/\delta)}{n_t(a)^+}}\right) =\order\left(\sqrt{AT\log(T/\delta)}\right). 
    \end{align*}
    Combining everything finishes the proof. 
\end{proof}


\section{Deep Hierarchical Bandits Extension}\label{app:deep hie}
In this subsection, we consider the deep hierarchical bandit setting. In this setting, there are $D$ agents making decisions in a fixed order: in each round $t$, Agent 1 first chooses an action $a^1_t\in[A]$. For $d=2, \ldots, D$, after receiving $(a^1_t, a^2_t, \ldots, a^{d-1}_t)$, Agent $d$ chooses an action $a_t^d\in[A]$.\footnote{For simplicity, we assume that the number of actions on all layers are the same and equal to $A$.} After all agents all choose an action, a reward $r_t\in[0,1]$ is generated based on the joint action $\aB_t\triangleq (a_t^1, \ldots, a_t^D)$ with its mean equal to $\mu_{\aB_t}$. As before, we assume that the first action is the best action in all layers, and therefore the goal of the agents is to have sub-linear regret with respect to the joint action $\1=(1,\ldots, 1)$. 


We propose \pref{alg:4.5} to solve this problem, which is based on the similar idea as \pref{alg:4.2}. At time $t$, Agent $d$ maintains the number of times an arm $\aB^{1:d}=(a^1,\dots,a^d)\in[A]^d$ has been visited: 
\beq\label{eq:4.12}
n_t^d(\aB^{1:d})=\sum_{s=1}^{t-1}\I\{\aB_s^{1:d}=\aB^{1:d}\}
\eeq
and the empirical mean of the same arm
\beq\label{eq:4.13}
\muh_t^d(\aB^{1:d})=\frac{1}{n_t^d(\aB^{1:d})^+}\sum_{s=1}^{t-1}r_s\I\{\aB_s^{1:d}=\aB^{1:d}\}.
\eeq
The bonus term of Agent $d$ for the arm $\aB^{1:d}$ is of order $\sqrt{\frac{A^{D-d}\log(A^DT/\dl)}{n_t^d(\aB^{1:d})^+}}$, which is again the average regret upper bound of its direct subordinate. As in \pref{sec: bandit}, such a design is to ensure that agent $d$'s optimistic value is not smaller than the value of the optimal arm (with high probability). 

\begin{algorithm}[t]
\caption{UCB for Deep Hierarchical Bandits}\label{alg:4.5}
\textbf{define}: $C_D\geq2$ and $C_d=6C_{d+1}+8$ for $d=D-1, \ldots, 1$. \\
\For{$t=1,\dots,T$}{
\For{$d=1,\dots,D$}{
Agent $d$ chooses $a_t^d\in\underset{a\in[A]}{\arg\max}\es\muh_t^d(a_t^1,\dots,a_t^{d-1},a)+C_d\sqrt{\frac{A^{D-d}\log(A^DT/\dl)}{n_t^d(a_t^1,\dots,a_t^{d-1},a)^+}}$, 
}
}
\end{algorithm}

In the following, we show the regret bounds of \pref{alg:4.5}. \pref{lem:4.6} and \pref{thm:4.7} are the gap-independent results, where \pref{lem:4.6} is a generalization of \pref{ass: U2's algorithm} and \pref{thm:4.7} is a generalization of \pref{thm: main theorem for bandit}. \pref{lem:4.8} and \pref{thm:4.9} are the gap-dependent results, where \pref{thm:4.9} is a generalization of \pref{thm: gap dependent bound}. In the bounds, the regret grows exponentially in the depth $D$; however, this comes from the high model complexity and is unavoidable.

\begin{lem}\label{lem:4.6}
For any $t\in[T]$, $d\in[D-1]$, $\aB^{1:d}=(a^1,\dots,a^d)\in[A]^d$,
\begin{equation}\label{eq:4.14}
\sum_{s=1}^t\I\{\aB_s^{1:d}=\aB^{1:d}\}\left[\mu_{(\aB^{1:d},\1_{D-d})}-r_s\right]
\leq
C_d\sqrt{A^{D-d}n_{t+1}^d(\aB^{1:d})\log(A^DT/\dl)}
\end{equation}
Furthermore, 
\begin{align}
    \sum_{s=1}^{t}  \left(\mu_{\1}-r_s\right) \leq (6C_1+8)\sqrt{A^Dt\log(A^DT/\dl)}.   \label{eq:case d=0}
\end{align}
\end{lem}
\begin{proof}\textbf{of \pref{lem:4.6}\ }\hypertarget{pf:lem:4.6}{}
\hfill\\
Notice that \pref{eq:case d=0} can be viewed as a case of $d=0$ in \pref{eq:4.14} by defining $\I\{\aB_s^{1:0}=\aB^{1:0}\}=1$ and $n_{t+1}(\aB^{1:d})=t$. Therefore, below we prove by induction from $d=D-1$ to $d=0$. \\
\hfill \\
\underline{Base $d=D-1$}:
\begin{align*}
&\qquad\sum_{s=1}^t\I\{\aB_s^{1:D-1}=a^{1:D-1}\}\left[\mu_{(\aB^{1:D-1},1)}-r_s\right]\\
&\leq C_{D-1}\sqrt{An_{t+1}^{D-1}(\aB^{1:D-1})\log(An_{t+1}^{D-1}(\aB^{1:D-1})/\dl)}
\leq C_{D-1}\sqrt{An_{t+1}^{D-1}(\aB^{1:D-1})\log(AT/\dl)}
\end{align*}
by the standard analysis of the UCB algorithm. 

\noindent\underline{Induction step}: assume the statement is true for $d+1$, we show that it holds for $d$ as well. Dividing both sides of the inequality in the induction hypothesis by $n_{t+1}^{d+1}(\aB^{1:d+1})^+$ to get
\beq\label{eq:indhyp}
\mu_{(\aB^{1:d+1},\1_{D-d-1})}\leq\muh_{t+1}^{d+1}(\aB^{1:d+1})+C_{d+1}\sqrt{\frac{A^{D-d-1}\log(A^DT/\dl)}{n_{t+1}^{d+1}(\aB^{1:d+1})^+}}.
\eeq
Now for the left hand side of the inequality for $d$, we have
\allowdisplaybreaks
\begin{align*}
&\qquad\sum_{s=1}^t\I\{\aB_s^{1:d}=\aB^{1:d}\}\left[\mu_{(\aB^{1:d},\1_{D-d})}-r_s\right]\\
&\leq\sum_{s=1}^t\I\{\aB_s^{1:d}=\aB^{1:d}\}\left\{\max_{a^{d+1}}\left[\muh_{t+1}^{d+1}(\aB_s^{1:d},a^{d+1})+C_{d+1}\sqrt{\frac{A^{D-d-1}\log(A^DT/\dl)}{n_{t+1}^{d+1}(\aB_s^{1:d},a^{d+1})^+}}\right]-r_s\right\}\\
&\leq\sum_{s=1}^t\I\{\aB_s^{1:d}=\aB^{1:d}\}\Bigg[\muh_{t+1}^{d+1}(\aB_s^{1:d+1})+C_{d+1}\sqrt{\frac{A^{D-d-1}\log(A^DT/\dl)}{n_{t+1}^{d+1}(\aB_s^{1:d+1})^+}}-\mu_{(\aB_s^{1:d+1},\1_{D-d-1})}\\
&\qquad\qquad+\mu_{(\aB_s^{1:d+1},\1_{D-d-1})}-r_s\Bigg],\\\numberthis\label{eq:themaininequality}
\end{align*}
where the first inequality follows from \eqref{eq:indhyp} and taking max, and the second inequality follows from the way $a_s^{d+1}$ is selected.

For any $s\in[T]$, $\aB^{1:d+1}\in[A]^{d+1}$, if $n_s^{d+1}(\aB^{1:d+1})\geq 1$, then we have
\begin{align*}
&\qquad\muh_s^{d+1}(\aB^{1:d+1})-\mu_{(\aB^{1:d+1},\1_{D-d-1})}\\
&\leq \frac{1}{n_s^{d+1}(\aB^{1:d+1})^+}\sum_{u=1}^{s-1}\I\{\aB_u^{1:d+1}=\aB^{1:d+1}\}\left[r_u-\mu_{(\aB^{1:d+1},\1_{D-d-1})}\right]\\
&\leq\frac{1}{n_s^{d+1}(\aB^{1:d+1})^+}\sum_{u=1}^{s-1}\I\{\aB_u^{1:d+1}=\aB^{1:d+1}\}\left[r_u-\mu_{\aB_u}\right]\\
&\leq\frac{2}{n_s^{d+1}(\aB^{1:d+1})^+}\sqrt{\sum_{u=1}^{s-1}\I\{\aB_u^{1:d+1}=\aB^{1:d+1}\}\log(1/\dl)}
\leq 2\sqrt{\frac{\log(1/\dl)}{n_s^{d+1}(\aB^{1:d+1})^+}}
\end{align*}
with probability $1-\dl$. This also trivially holds if $n_s^{d+1}(\aB^{1:d+1})=0$. Then, the inequality
\beq\label{eq:mu_bound}
\muh_s^{d+1}(\aB^{1:d+1})-\mu_{(\aB^{1:d+1},\1_{D-d-1})}
\leq2\sqrt{\frac{\log(A^{d+1}T/\dl)}{n_s^{d+1}(\aB^{1:d+1})^+}}
\eeq
holds for all $s\in[T]$ and all $\aB^{1:d+1}\in[A]^{d+1}$ simultaneously by the union bound. Hence, the sum of the first three terms of \eqref{eq:themaininequality} can be bounded by
\begin{align*}
&\qquad\sum_{s=1}^t\I\{\aB_s^{1:d}=\aB^{1:d}\}\left[\muh_t^{d+1}(\aB_s^{1:d+1})+C_{d+1}\sqrt{\frac{A^{D-d-1}\log(A^DT/\dl)}{n_{t+1}^{d+1}(\aB_s^{1:d+1})^+}}-\mu_{(\aB_s^{1:d+1},\1_{D-d-1})}\right]\\
&\leq\sum_{s=1}^t\I\{\aB_s^{1:d}=\aB^{1:d}\}\left[2\sqrt{\frac{\log(A^{d+1}T/\dl)}{n_{t+1}^{d+1}(\aB_s^{1:d+1})^+}}+C_{d+1}\sqrt{\frac{A^{D-d-1}\log(A^DT/\dl)}{n_{t+1}^{d+1}(\aB_s^{1:d+1})^+}}\right]\\
&\leq(C_{d+1}+2)\sqrt{A^{D-d-1}\log(A^DT/\dl)}\cdot\sum_{s=1}^t\I\{\aB_s^{1:d}=\aB^{1:d}\}\sqrt{\frac{1}{n_s^{d+1}(\aB_s^{1:d+1})^+}}\\
&\leq2(C_{d+1}+2)\sqrt{A^{D-d}n_{t+1}^d(\aB^{1:d})\log(A^DT/\dl)}.
\end{align*}
On the other hand, the last two terms sum up to
\begin{align*}
&\qquad\sum_{s=1}^t\I\{\aB_s^{1:d}=\aB^{1:d}\}\left[\mu_{(\aB_s^{1:d+1},\1_{D-d-1})}-r_s\right]\\
&=\sum_{a^{d+1}\in[A]}\sum_{s=1}^t\I\{\aB_s^{1:d+1}=\aB^{1:d+1}\}\left[\mu_{(\aB_s^{1:d+1},\1_{D-d-1})}-r_s\right]\\
&\leq\sum_{a^{d+1}\in[A]}C_{d+1}\sqrt{A^{D-d-1}n_{t+1}^{d+1}(\aB^{1:d})\log(A^DT/\dl)}\\
&\leq C_{d+1}\sqrt{A^{D-d-1}\cdot\left[\sum_{a^{d+1}\in[A]}n_{t+1}^{d+1}(\aB^{1:d})\right]\log(A^DT/\dl)}\\
&\leq C_{d+1}\sqrt{A^{D-d}n_{t+1}^d(\aB^{1:d})\log(A^DT/\dl)}.
\end{align*}
Combining both parts, with probability $1-3\dl:=1-\dl'$, we have
\begin{align*}
&\qquad\sum_{s=1}^t\I\{\aB_s^{1:d}=\aB^{1:d}\}\left[\mu_{(\aB^{1:d},\1_{D-d})}-r_s\right]\\
&\leq[2(C_{d+1}+2)+C_{d+1}]\sqrt{A^{D-d}n_{t+1}^d(\aB^{1:d})\log(A^DT/\dl)}\\
&\leq\sqrt{3}(3C_{d+1}+4)\sqrt{A^{D-d}n_{t+1}^d(\aB^{1:d})\log(A^DT/\dl')}\\
&\leq(6C_{d+1}+8)\sqrt{A^{D-d}n_{t+1}^d(\aB^{1:d})\log(A^DT/\dl')} \\
&\leq C_d\sqrt{A^{D-d}n_{t+1}^d(\aB^{1:d})\log(A^DT/\dl')}
\end{align*}
whenever $C_d\geq6C_{d+1}+8$.

Notice that when $d=0$, the same arguments follow except that $\I\{\aB_s^{1:d}=\aB^{1:d}\}$ degenerates to $1$, $n_{t+1}^d(\aB^{1:d})$ degenerates to $t$, and the inequality will end with $(6C_1+8)\sqrt{A^Dt\log(A^DT/\dl')}$.
\end{proof}

\begin{thm}\label{thm:4.7}
For the deep hierarchical bandit problem, with probability of at least $1-\dl$, \pref{alg:4.5} achieves the regret bound of
\beq\label{eq:4.15}
\sum_{t=1}^T\left(\mu_{\1_D}-\mu_{\aB_t}\right)\leq O\left(\sqrt{A^DDT\log(AT/\dl)}\right).
\eeq
\end{thm}
\begin{proof}
By \pref{lem:4.6} we have $\sum_{t=1}^T\left(\mu_{\1_D} - r_t \right)=\order\big(\sqrt{A^D DT\log(AT/\delta)}\big)$. Further using the fact that $\sum_{t=1}^T \left(r_t - \mu_{\aB_t}\right) = \order\big(\sqrt{T\log(T/\delta)}\big)$ finishes the proof. 
\end{proof}

\begin{lem}\label{lem:4.8}
For any $d\in[D]$, let $a^d\in[A]$ be a sub-optimal arm of agent $d$ given that $\aB^{1:d-1}$ is chosen by the first $d-1$ agents, then with probability at least $1-2\dl$,
\beq\label{eq:4.16}
\sum_{t=1}^T\I\{\aB_t^{1:d}=\aB^{1:d}\}\leq\frac{4(C_d+2)^2A^{D-d}\log(A^DT/\dl)}{\left[\mu_{(\aB^{1:d-1},\1_{D-d+1})}-\mu_{(\aB^{1:d},\1_{D-d})}\right]^2}.
\eeq
\end{lem}

\begin{proof}\textbf{of \pref{lem:4.8}\ }\hypertarget{pf:lem:4.8}{}
Suppose at time $t+1$,
\beq\label{eq:ncond}
n_{t+1}^d(\aB^{1:d})\geq\frac{4(C_d+2)^2A^{D-d}\log(A^DT/\dl)}{\left[\mu_{(\aB^{1:d-1},\1_{D-d+1})}-\mu_{(\aB^{1:d},\1_{D-d})}\right]^2}.
\eeq
Then with probability at least $1-\dl$,
\begingroup
\allowdisplaybreaks
\begin{align*}
&\eqsp\muh_{t+1}^{d}(\aB^{1:d})+C_d\sqrt{\frac{A^{D-d}\log(A^DT/\dl)}{n_{t+1}^d(\aB^{1:d})^+}}\\
&\leq\mu_{(\aB^{1:d},\1_{D-d})}+2\sqrt{\frac{\log(A^dT/\dl)}{n_{t+1}^d(\aB^{1:d})^+}}+C_d\sqrt{\frac{A^{D-d}\log(A^DT/\dl)}{n_{t+1}^d(\aB^{1:d})^+}}\tag{by \eqref{eq:mu_bound}}\\
&\leq\mu_{(\aB^{1:d},\1_{D-d})}+(C_d+2)\sqrt{\frac{A^{D-d}\log(A^DT/\dl)}{n_{t+1}^d(\aB^{1:d})^+}}\\
&\leq\mu_{(\aB^{1:d},\1_{D-d})}+\frac{\mu_{(\aB^{1:d-1},\1_{D-d+1})}-\mu_{(\aB^{1:d},\1_{D-d})}}{2}
<\mu_{(\aB^{1:d-1},\1_{D-d+1})}
\end{align*}
\endgroup
On the other hand, by \pref{lem:4.6} with $(t,d,(\aB^{1:d-1},1))$, with probability at least $1-\dl$,
\[
\mu_{(\aB^{1:d-1},\1_{D-d+1})}\leq\muh_t^{d}(\aB^{1:d-1},1)+C_d\sqrt{\frac{A^{D-d}\log(A^DT/\dl)}{n_{t+1}^d(\aB^{1:d-1},1)^+}}.
\]
Hence, by the design of \pref{alg:4.5}, given that the first $d-1$ agents choose $\aB^{1:d-1}$, agent $d$ will not choose the sub-optimal arm $a^d\neq1$ over the first arm before time $T$ when the condition \eqref{eq:ncond} is satisfied.
\end{proof}

\begin{thm}\label{thm:4.9}
For the deep hierarchical bandit problem, with probability of at least $1-\dl$, \pref{alg:4.5} achieves the regret bound of
\beq\label{eq:4.17}
\sum_{t=1}^T[\mu_{\1_D}-\mu_{\aB_t}]\leq\sum_{d=1}^D\sum_{\aB^{1:d}\neq\1_d}O\left(\frac{A^{D-d}D\log(ADT/\dl)}{\mu_{(\aB^{1:d-1},\1_{D-d+1})}-\mu_{(\aB^{1:d},\1_{D-d})}}\right).
\eeq
\end{thm}

\begin{proof}\textbf{of \pref{thm:4.9}\ } \hypertarget{pf:thm:4.9}{}
From \pref{lem:4.8} and the union bound, the inequality \eqref{eq:4.17} will hold simultaneously for all $d\in[D]$ and $\aB^{1:d}\in[A]^d$ with probability at least $1-2A^DD\dl$. Hence, with probability at least $1-2A^DD\dl$,
\allowdisplaybreaks
\begin{align*}
&\eqsp\sum_{t=1}^T[\mu_{\1_D}-\mu_{\aB_t}]
=\sum_{\aB\neq\1_D}\left(\mu_{\1_D}-\mu_\aB\right)\cdot n_{T+1}^D(\aB)\\
&=\sum_{\aB\neq\1_D}\sum_{d=1}^D\left\{\left[\mu_{(\aB^{1:d-1},\1_{D-d+1})}-\mu_{(\aB^{1:d},\1_{D-d})}\right]\cdot n_{T+1}^D(\aB)\right\}\\
&=\sum_{d=1}^D\sum_{\aB^{1:d}\neq\1_d}\left\{\left[\mu_{(\aB^{1:d-1},\1_{D-d+1})}-\mu_{(\aB^{1:d},\1_{D-d})}\right]\cdot\left[\sum_{\aB^{d+1:D}}n_{T+1}^D(\aB)\right]\right\}\\
&=\sum_{d=1}^D\sum_{\aB^{1:d}\neq\1_d}\left\{\left[\mu_{(\aB^{1:d-1},\1_{D-d+1})}-\mu_{(\aB^{1:d},\1_{D-d})}\right]\cdot\left[\sum_{\aB^{d+1:D}}n_{T+1}^D(\aB)\right]\right\}\\
&=\sum_{d=1}^D\sum_{\aB^{1:d}\neq\1_d}\left\{\left[\mu_{(\aB^{1:d-1},\1_{D-d+1})}-\mu_{(\aB^{1:d},\1_{D-d})}\right]\cdot\left[\sum_{t=1}^T\I\{\aB_t^{1:d}=\aB^{1:d}\}\right]\right\}\\
&\leq\sum_{d=1}^D\sum_{\aB^{1:d}\neq\1_d}\left\{\left[\mu_{(\aB^{1:d-1},\1_{D-d+1})}-\mu_{(\aB^{1:d},\1_{D-d})}\right]\cdot\frac{4(C_d+2)^2A^{D-d}\log(A^DT/\dl)}{\left[\mu_{(\aB^{1:d-1},\1_{D-d+1})}-\mu_{(\aB^{1:d},\1_{D-d})}\right]^2}\right\}\\
&=\sum_{d=1}^D\sum_{\aB^{1:d}\neq\1_d}\frac{4(C_d+2)^2A^{D-d}\log(A^DT/\dl)}{\mu_{(\aB^{1:d-1},\1_{D-d+1})}-\mu_{(\aB^{1:d},\1_{D-d})}}.
\end{align*}
Letting $\dl'=2A^DD\dl$ gives the claim.
\end{proof}

\section{Auxiliary Lemmas}\label{app: aux}

\begin{lemma}[Hoeffding-Azuma inequality]\label{lem: UCB concentration}
    Let $\calF_0\subset \calF_1 \subset\cdots \subset \calF_{n}$ be a filtration, and $X_1, \ldots, X_n$ be real random variables such that $X_i$ is $\calF_i$-measurable, $\E[X_i|\calF_{i-1}]=0$, $|X_i|\leq b$ for some fixed $b\geq 0$. Furthermore, let $\{y_i\}_{i=1}^n$ be a fixed sequence. Then with probability at least $1-\delta$, 
\begin{align*}
    \sum_{i=1}^n y_iX_i \leq b\sqrt{2\left(\sum_{i=1}^n y_i^2\right) \log(1/\delta)}. 
\end{align*}
\end{lemma}


\begin{lemma}[\cite{weissman2003inequalities, jaksch2010near}]\label{lem: Weissman}
    Let $\hatp(\cdot)\in \mathbb{R}_+^{d}$ be the empirical over $d$ distinct events from $n$ samples, and let $p(\cdot)$ be the true underlying distribution. Then with probability at least $1-\delta$, 
    \begin{align*}
        \left\|\hatp(\cdot)-p(\cdot)\right\|_1 \leq \sqrt{\frac{2d\log(T/\delta)}{n}}.
    \end{align*}
\end{lemma}

\begin{lemma}[Lemma 4.1 of \citep{UCB-H_2018}]\label{lem: lemmas regarding alpha}
For a positive integer $\tau$, define $\alpha^i_\tau=\alpha_i \Pi_{j=i+1}^\tau (1-\alpha_j)$ for $1\leq i\leq \tau$ and $\alpha^0_\tau=\Pi_{j=1}^\tau (1-\alpha_j)$ where $\alpha_\tau=\frac{H+1}{H+\tau}$\,. Then the following hold: 
\begin{enumerate}
    \item $\frac{1}{\sqrt{\tau}}\leq \sum_{i=1}^\tau \frac{\alpha^i_\tau}{\sqrt{i}} \leq \frac{2}{\sqrt{\tau}}$ for all $\tau\geq 1$.
    \item $\max_{i\in[\tau]}\alpha^i_\tau \leq \frac{2H}{\tau}$ and $\sum_{i=1}^\tau (\alpha^i_\tau)^2\leq \frac{2H}{\tau}$ for all $\tau\geq 1$. 
    \item $\sum_{\tau=i}^\infty \alpha^i_\tau = 1+\frac{1}{H}$ for all $i\geq 1$. 
\end{enumerate}

\end{lemma}

\begin{lemma}\label{lem: alpha concentration}
    Let $\alpha^i_n$ be defined as in \pref{lem: lemmas regarding alpha}. Let $\calF_0\subset \calF_1 \subset\cdots \subset \calF_{n}$ be a filtration, and $p_1, \ldots, p_n$ be distribution over $\calS$ where $p_i$ is deterministic given $\calF_{i-1}$. Suppose that $s_i\in\calS$ is drawn from $p_i$. Then with probability at least $1-\delta$, for all $V:~\calS\rightarrow [0, 1]$, 
     \begin{align*}
        \left|\sum_{i=1}^n \alpha^i_n V(s_i) - \sum_{i=1}^n \alpha^i_n \sum_{s}p_i(s)V(s)\right| \leq 3\sqrt{\frac{SH}{n}\log(4n/\delta)}. 
    \end{align*}
\end{lemma}

\begin{proof}
    Consider the following $\epsilon$-cover for the space of $V$: 
    \begin{align*}
        \calV = \left\{V: \calS\rightarrow \{0, \epsilon, 2\epsilon, \ldots, 1\}\right\}
    \end{align*}
    For every fixed $V\in\calV$, we have with probability at least $1-\delta'$, 
    \begin{align*}
        \left|\sum_{i=1}^n \alpha^i_n V(s_i) - \sum_{i=1}^n \alpha^i_n \sum_{s}p_i(s)V(s)\right| \leq \sqrt{2\left(\sum_{i=1}^n (\alpha^i_n)^2\right)\log(2/\delta')} \leq 2\sqrt{\frac{H}{n}\log(2/\delta')}. 
    \end{align*}
    By an union bound, the above holds for all $V\in\calV$ with probability at least $1-|\calV|\delta'$. 

    For any $V: \calS\rightarrow [0,1]$, there is a $\tilde{V}\in\calV$ such that $|V(s)-\tilde{V}(s)|\leq \frac{\epsilon}{2}$ for all $s$. 
    Thus, with probability $1-|\calV|\delta'$, we have 
    \begin{align*}
        \left|\sum_{i=1}^n \alpha^i_n V(s_i) - \sum_{i=1}^n \alpha^i_n \sum_{s}p_i(s)V(s)\right| \leq 2\sqrt{\frac{H}{n}\log(2/\delta')} + \epsilon.  
    \end{align*}
    Picking $\epsilon=\frac{1}{n}$ (which implies $|\calV|=\left(\frac{1}{\epsilon}+1\right)^S\leq \left(\frac{2}{\epsilon}\right)^S=(2n)^S$), and $\delta'=\frac{\delta}{|\calV|}$, the right-hand side above can be upper bounded by 
    \begin{align*}
        2\sqrt{\frac{H}{n}\log(2(2n)^S/\delta)} + \frac{1}{n} \leq 3\sqrt{\frac{SH}{n}\log(4n/\delta)}. 
    \end{align*}

    
\end{proof}

\end{document}